\documentclass{article}

\PassOptionsToPackage{square,sort,comma,numbers}{natbib}
\usepackage[final]{neurips_2021}

\usepackage[utf8]{inputenc} 
\usepackage[T1]{fontenc}    
\usepackage{hyperref}       
\usepackage{url}            
\usepackage{booktabs}       
\usepackage{amsfonts}       
\usepackage{nicefrac}       
\usepackage{microtype}      
\usepackage{xcolor}         

\makeatletter
\let\@fnsymbol\@arabic
\makeatother
    
\usepackage{tikz}
\usetikzlibrary{decorations}
\usetikzlibrary{decorations.pathreplacing}
\usepackage{amsthm}
\usepackage{amsmath}
\usepackage{amssymb}
\usepackage{tikz}

\newtheorem{theorem}{Theorem}

\usepackage{bbm}

\newtheorem*{remark}{Remark}

\usepackage{cleveref}
\Crefformat{figure}{#2Fig.~#1#3}
\Crefformat{theorem}{#2Thm.~#1#3}
\Crefformat{definition}{#2Def.~#1#3}
\Crefformat{algorithm}{#2Alg.~#1#3}
\Crefformat{section}{#2Sec.~#1#3}
\usepackage{xcolor}
\usepackage{enumitem}
\usepackage[toc,page]{appendix}
\usepackage{cancel}
\usepackage[acronym,smallcaps]{glossaries}
\glsdisablehyper
\newacronym{OSE}{ose}{\textit{oblivious subspace embedding}}
\newacronym{SRHT}{SRHT}{Subsampled Randomized Hadamard Transform}
\newacronym{SJLT}{SJLT}{Sparse Johnson Lindenstrauss Transform}
\usepackage{subcaption}
\newcommand{\norm}[1]{\left\lVert#1\right\rVert}
\usepackage{wrapfig}
\usepackage{algorithm}
\usepackage[noend]{algpseudocode}
\algrenewcommand{\algorithmiccomment}[1]{\hskip3em// \textit{#1}}
\DeclareMathSymbol{\mathdblquotechar}{\mathalpha}{letters}{`"}
\usepackage{multirow}
\usepackage{makecell}
\usepackage{bm}
\usepackage{arydshln}
\newcommand{\indep}{\perp \!\!\! \perp}
\DeclareMathOperator*{\argmin}{arg\,min}

\title{Higher Order Kernel Mean Embeddings to Capture Filtrations of Stochastic Processes}

\author{
  Cristopher Salvi
  \thanks{University of Oxford \& The Alan Turing Institute}\\
  \texttt{cristopher.salvi@maths.ox.ac.uk} \\

  \And
  Maud Lemercier \thanks{University of Warwick \& The Alan Turing Institute}\\
  \texttt{maud.lemercier@warwick.ac.uk}\\
  
  \And
  Chong Liu \thanks{University of Oxford}\\
  \texttt{chong.liu@maths.ox.ac.uk} \\
  
  \And
  Blanka Horvath \thanks{King's College London, The Alan Turing Institute \& Technical University of Munich}\\
  \texttt{blanka.horvath@kcl.ac.uk} \\
  
  \And
  Theodoros Damoulas \footnotemark[2] \\
  \texttt{t.damoulas@warwick.ac.uk} \\
  
  \And
  Terry Lyons \footnotemark[1] \\
  \texttt{terry.lyons@maths.ox.ac.uk}
  
}

\begin{document}

\maketitle

\begin{abstract}
  Stochastic processes are random variables with values in some space of paths. However, reducing a stochastic process to a path-valued random variable ignores its \emph{filtration}, i.e. the flow of information carried by the process through time. By conditioning the process on its filtration, we introduce a family of \emph{higher order kernel mean embeddings} (KMEs) that generalizes the notion of KME and captures additional information related to the filtration. We derive empirical estimators for the associated \emph{higher order maximum mean discrepancies} (MMDs) and prove consistency. We then construct a filtration-sensitive kernel two-sample test able to pick up information that gets missed by the standard MMD test. In addition, leveraging our higher order MMDs we construct a family of universal kernels on stochastic processes that allows to solve real-world calibration and optimal stopping problems in quantitative finance (such as the pricing of American options) via classical kernel-based regression methods. Finally, adapting existing tests for conditional independence to the case of stochastic processes, we design a causal-discovery algorithm to recover the causal graph of structural dependencies among interacting bodies solely from observations of their multidimensional trajectories.
\end{abstract}

\section{Introduction}

The idea of embedding probability distributions into a reproducing kernel Hilbert space (RKHS) via kernel mean embeddings (KMEs) has become ubiquitous in many areas of statistics and data science such as hypothesis testing \cite{gretton2012kernel,zhang2012kernel}, non-linear regression \cite{scholkopf2002learning, hofmann2008kernel}, distribution regression \cite{szabo2016learning, lemercier2021distribution} etc. Despite strong progress in the study of KMEs, most of the examples considered in the literature tend to focus on random variables supported on some finite (possibly high) dimensional euclidean spaces like $\mathbb{R}^d$. The study of KMEs for function-valued random variables has been largely ignored.   

Stochastic processes are random variables with values in some space of paths. However, reducing a stochastic process to a path-valued random variable ignores its \emph{filtration}, which can be informally thought of as the \emph{flow of information carried by the process through time}. A question that naturally emerges from the study of many random, time-evolving systems like financial markets is how does the information available up to present time affect the future evolution of the system?

Formally, this question can be addressed by conditioning a process on its filtration (\Cref{sec:cond_kme} and \ref{sec:cond_filtration}). In this paper we introduce a family of \emph{higher order KMEs} that generalizes the notion of KME to capture additional, filtration-related information (\Cref{sec:second_order} and \ref{sec:higher_order}). In view of concrete applications, we derive empirical estimators for the associated \emph{higher order MMDs} and use one of them to construct a filtration-sensitive kernel two-sample test (\Cref{sec:sample_test}) demonstrating with simulated data its ability to capture information that otherwise gets missed by the standard MMD test (\Cref{sec:hypo}). Furthermore, we construct a family of universal kernels on stochastic processes (\Cref{sec:dr}) that allows to solve challenging, real-world optimisation problems in quantitative finance such as the pricing of American options via classical kernel-based regression methods (\Cref{sec:dr_finance}). Finally, we adapt existing tests for conditional independence to the case of stochastic processes in order to design a causal-discovery algorithm able to recover the causal graph of structural dependencies among interacting bodies solely from observations of their multidimensional trajectories (\Cref{sec:causality}). 

\begin{figure}
    \centering
    \includegraphics[scale=0.9]{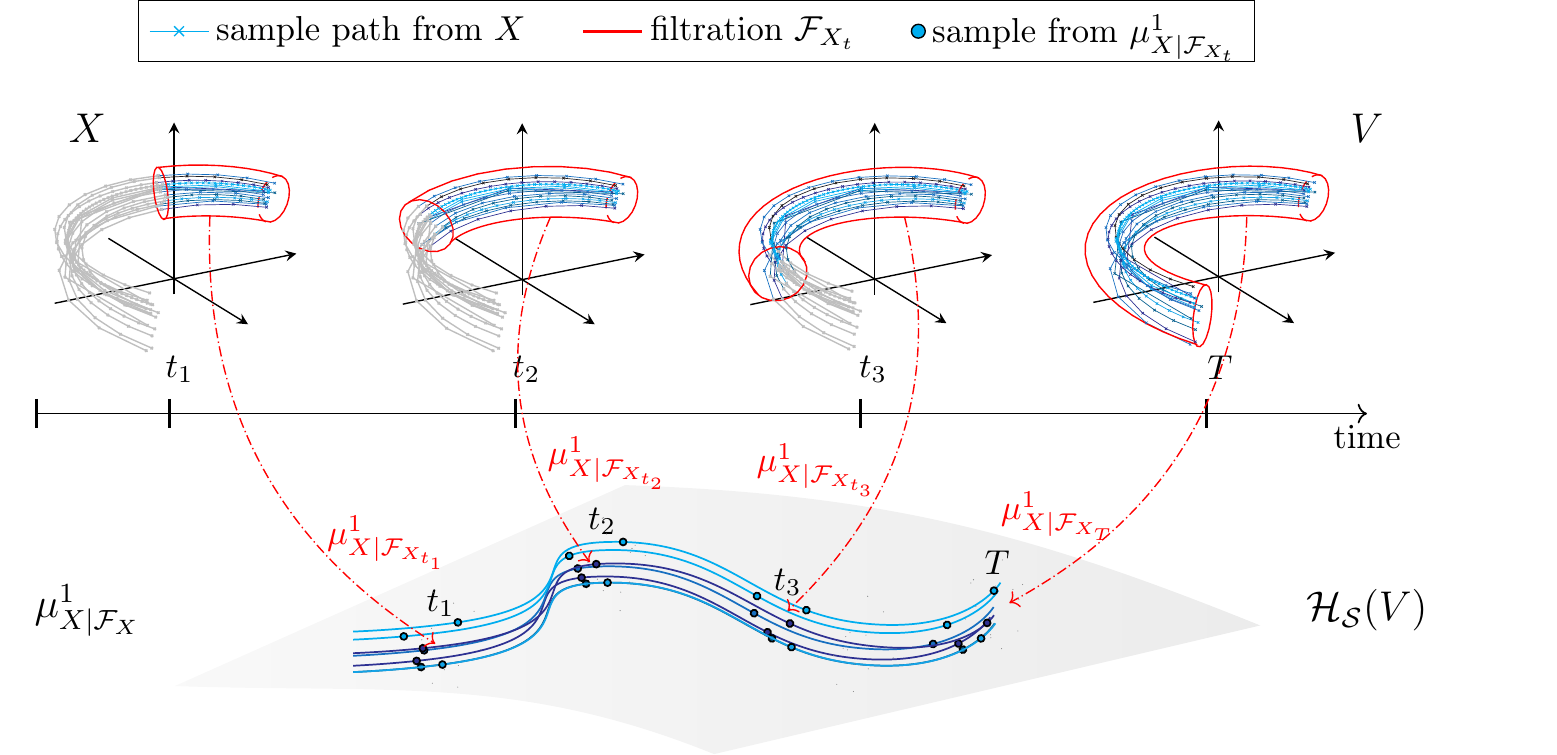}
    \caption{\small Schematic overview of the construction of the $1^{\text{st}}$ order predictive KME $\mu^1_{X|\mathcal{F}_X}$ (\Cref{sec:cond_filtration}). Here $X$ is a stochastic process with sample paths taking their values in $V$. The red contours indicate the portion of its filtration $\mathcal{F}_X$ upon which the conditioning is applied, i.e. the available information about $X$ from start up to time $t$. As explained in \Cref{sec:cond_filtration}, the $1^{\text{st}}$ order predictive KME  $\mu^1_{X|\mathcal{F}_X}$ is a path whose value at time $t$ is a $\mathcal{H}_{\mathcal{S}}(V)$-valued random variable representing the law of $X$ conditioned on its filtration $\mathcal{F}_{X_t}$. Equivalently $\mu^1_{X|\mathcal{F}_X}$ is a stochastic process with sample paths taking their values in $\mathcal{H}_{\mathcal{S}}(V)$.}
    \label{fig:main}
\end{figure}

\subsection{Related work}

The notion of conditioning is a powerful probabilistic tool allowing to understand possibly complex, non-linear interactions between random variables. As their unconditional counterparts, conditional distributions can also be embedded into RKHSs \cite{song2013kernel}. Recently, conditional KMEs have received increased attention, especially in the context of graphical models \cite{song2010nonparametric}, state-space models \cite{fukumizu2013kernel}, dynamical systems \cite{song2009hilbert}, causal inference \cite{mitrovic2018causal, tillman2009nonlinear,sun2007kernel}, two-sample and conditional independence hypothesis testing \cite{fukumizu2007kernel,sun2007kernel,park2020measure} and others. Embeddings of distributions via KMEs have also shown their success in the context of distribution regression (DR), which is the task of learning a function mapping a collection of samples from a probability distribution to scalar targets \cite{law2018bayesian, muandet2012learning, smola2007hilbert}. More recently, a framework for DR that addresses the setting where inputs are sample paths from an underlying stochastic process is proposed in \cite{lemercier2021distribution}. The authors make extensive use of the \emph{signature transform} \cite{lyons2014rough, bonnier2019deep} and of the \emph{signature kernel} \cite{kiraly2019kernels, cass2020computing}, two well established tools in stochastic analysis.

When it comes to stochastic processes, it was first shown in \cite{Aldous81adapt} that weak convergence of random variables does not always account for the information contained in the filtration, as highlighted by means of numerous numerical examples in \cite{Pflug2012multi, Backhoff2021adapted}. This limitation is addressed in \cite{Aldous81adapt, Hoover84adapt} through the construction of a sequence of so--called \emph{adapted topologies}\footnote{We say that a sequence of random variables $\{X_n\}_{n \in \mathbb{N}}$ converges to a random variable $X$ in the topology $\tau$ if and only if for every $\tau$-open neighbourhood $\mathbb{U}$ of X there exists $N \in \mathbb{N}$ such that $X_n \in \mathbb{U}$ as soon as $n\geq N$.} $(\tau_n)_{n \geq 1}$ that become progressively finer\footnote{A topology $\tau_1$ is said to be finer than a topology $\tau_2$ if every $\tau_2$-open set is also $\tau_1$-open.} as $n$ increases (with $\tau_1$ coinciding with the weak topology). In particular, higher order adapted topologies are shown to capture more filtration-related information than their weak counterpart. This characteristic becomes relevant for example in some optimal stopping problems such as the pricing of American options, where the pricing function can be shown to be discontinuous with respect to the weak topology, but is continuous with respect to the second order adapted topology\footnote{The second order adapted topology $\tau_2$ is equivalent to the \emph{adapted Wasserstein distance} \cite{Backhoff2019adapted}.} \cite{Pflug2012multi, backhoff2020all, Backhoff2019adapted}. Leveraging properties of the signature transform, it has been shown that adapted topologies are intimately related to a family of higher order MMDs \cite{bonnier2020adapted}. However, providing empirical estimators for these discrepancies that can be deployed on real-world tasks remains a challenge. In this paper we propose to address this challenge by presenting an alternative construction to this higher order MMDs using the language of kernels and KMEs. The results in \cite{bonnier2020adapted} serve as a strong theoretical background for the present paper. 

\section{Preliminaries}

We begin with a brief summary of tools from stochastic analysis needed to define higher order KMEs. 
Let $\mathcal{X}(\mathbb{R}^d) = \{x : [0,T] \to \mathbb{R}^d\}$ a compact set of continuous, piecewise linear, $\mathbb{R}^d$-valued paths defined over a common time interval $[0,T]$, obtained for example by linearly interpolating a multivariate time series. For any path $x \in \mathcal{X}(\mathbb{R}^d)$ we denote its $k^{\text{th}}$ coordinate by $x^{(k)} : [0,T] \to \mathbb{R}$, for $k \in \{1,\ldots,d\}$. More generally we denote by $\mathcal{X}(V) = \{x : [0,T] \to V\}$ a compact set of continuous, piecewise linear paths with values in a Hilbert space $V$ with a countable basis. 

\subsection{The signature transform}

The \textit{signature transform} $\mathcal{S}:\mathcal{X}(V) \to H_{\mathcal{S}}(V)$ is a \emph{feature map} defined for any path $x \in \mathcal{X}(V)$ as the following infinite collection of statistics
\begin{equation}
    \mathcal{S}(x) = \left(1, \left\{\mathcal{S}(x)^{(k_1)}\right\}_{k_1=1}^d, \left\{\mathcal{S}(x)^{(k_1,k_2)}\right\}_{k_1,k_2=1}^d, \ldots \right)
\end{equation}
where each term is a real number equal to the iterated integral 
\begin{equation}\label{eqn:sig}
\mathcal{S}(x)^{(k_1,\ldots,k_j)} =  \underset{0<s_1<\ldots<s_j<T}{\int \ldots \int} d x^{(k_1)}_{s_1} \ldots dx^{(k_j)}_{s_j}
\end{equation}
The signature \emph{feature space} $H_{\mathcal{S}}(V)$ is defined as the following direct sum of tensor powers of $V$
\begin{equation}
    H_{\mathcal{S}}(V) = \bigoplus_{k=0}^\infty V^{\otimes k} = \mathbb{R} \oplus V\oplus (V)^{\otimes 2} \oplus \ldots 
\end{equation}
where $\otimes$ denotes the standard tensor product of vector spaces \cite{lyons1998differential, lyons2014rough}.

\subsection{The signature kernel}

Because $V$ is Hilbert $H_{\mathcal{S}}(V)$ is also Hilbert \cite{kiraly2019kernels}. The \emph{signature kernel} $k_{\mathcal{S}}:\mathcal{X}(V) \times \mathcal{X}(V) \to \mathbb{R}$ is a characteristic kernel defined for any pair of paths $x,y \in \mathcal{X}(V)$ as the following inner product
\begin{align}\label{eq:kernel}
    k_{\mathcal{S}}(x,y) = \left\langle \mathcal{S}(x), \mathcal{S}(y) \right\rangle_{H_{\mathcal{S}}(V)} 
\end{align} 
The recent article \cite{cass2020computing} establishes a surprising connection between the signature kernel and a certain class of partial differential equations (PDEs), culminating in the following kernel trick for $k_{\mathcal{S}}$.

\begin{theorem}\cite[Thm. 2.5]{cass2020computing}\label{thm:sig_PDE}
For any $x,y \in \mathcal{X}(V)$ the signature kernel satisfies the equation $k_{\mathcal{S}}(x,y)=u_{x,y}(T,T)$, where $u_{x,y} : [0,T] \times [0,T] \to \mathbb{R}$ is the solution of the hyperbolic PDE
\begin{equation}\label{eq:sig_PDE}
    \frac{\partial^2 u_{x,y}}{\partial s \partial t} = \left\langle \dot x_s, \dot y_t \right\rangle_V u_{x,y}
\end{equation}
with boundary conditions $u_{x,y}(0,\cdot)=u_{x,y}(\cdot,0)=1$ and where $\dot z_s=\frac{d z_r}{dr}\big|_{r=s}$.
\end{theorem}
Hence, evaluating $k_{\mathcal{S}}$ at a pair of paths $(x,y)$ is equivalent to solving the PDE (\ref{eq:sig_PDE}); in this paper we solve PDEs numerically via a finite difference scheme (see \Cref{sec:algo} for additional details). In what follows, we denote by $\mathcal{H}_{\mathcal{S}}(V)$ the RKHS associated to $k_{\mathcal{S}}$.

\subsection{Stochastic processes and filtrations}

We take $(\Omega, \mathcal{F}, \mathbb{P})$ as the underlying probability space. A \emph{(discrete time) stochastic process} $X$ is a random variable with values on $\mathcal{X}(V)$. We denote by $\mathbb{P}_X=\mathbb{P} \circ X^{-1}$ the \emph{law} of $X$. Assuming the integrability condition $\mathbb{E}_X[k_\mathcal{S}(X,X)] < \infty$, the \emph{$1^{\text{st}}$ order kernel mean embedding} (KME) of $X$ is defined as\footnote{The $1^{\text{st}}$ order KME is the standard KME with the signature kernel $k_{\mathcal{S}}$.} the following point in $\mathcal{H}_{\mathcal{S}}(V)$
\begin{equation}
    \mu^{1}_X = \mathbb{E}_X[k_{\mathcal{S}}(X, \cdot)] = \int_{x \in \mathcal{X}(V)} k_{\mathcal{S}}(x,\cdot)\mathbb{P}_X(dx)
\end{equation}
Accordingly, given two stochastic processes $X,Y$, their \emph{$1^{\text{st}}$ order maximum mean discrepancy} (MMD) is the standard MMD distance with kernel $k_\mathcal{S}$ given by the following expression
\begin{equation}
    \mathcal{D}^1_{\mathcal{S}}(X,Y) = \norm{\mu^{1}_X - \mu^{1}_Y}_{\mathcal{H}_{\mathcal{S}}(V)}
\end{equation}
Because the signature kernel $k_\mathcal{S}$ is characteristic, it is a classical result \cite{gretton2012kernel, chevyrev2018signature} that the $1^{\text{st}}$ order MMD is a sufficient statistics to distinguish between the laws of $X$ and $Y$, in other words 
\begin{equation}
    \mathcal{D}^1_{\mathcal{S}}(X,Y) = 0 \iff \mathbb{P}_X = \mathbb{P}_Y
\end{equation}
Despite the fact that stochastic processes are path-valued random variables, they encode a much richer structure compared to standard $\mathbb{R}^d$-valued random variables, that goes well beyond their laws. This additional structure is described mathematically by the concept of \emph{filtration} of a process $X$, defined as the following family of $\sigma$-algebras 
\begin{equation}
    \mathcal{F}_X=(\mathcal{F}_{X_t})_{t \in [0,T]},
\end{equation}
where for any $t \in [0,T]$, $\mathcal{F}_{X_t}$ is the $\sigma$-algebra generated by the variables $\{X_s\}_{s \in [0,t]}$. Note that $\mathcal{F}_X$ is totally ordered in the sense that $\mathcal{F}_{X_s} \subset \mathcal{F}_{X_t} \text{ for all } s < t$, which naturally explains why filtrations are good mathematical descriptions to model the flow information carried by the process $X$.

In the next section, we will present our main findings and introduce a family of \emph{higher order KMEs} and corresponding \emph{higher order MMDs} as generalizations of the standard KME and MMD respectively. We will do so by conditioning stochastic processes on elements of their filtrations. 


\section{Higher order kernel mean embeddings}\label{sec:main_sec}

We begin by describing how KMEs can be extended to conditional laws of stochastic processes. 

\subsection{Conditional kernel mean embeddings for stochastic processes}\label{sec:cond_kme}


Let $X,Y$ be two stochastic processes. For a given path $x \in \mathcal{X}(V)$, define the \emph{$1^{\text{st}}$ order conditional kernel mean embeddings} $\mu^{1}_{Y \mid X=x} \in \mathcal{H}_{\mathcal{S}}(V)$ and $\mu^{1}_{Y \mid X} : \mathcal{H}_{\mathcal{S}}(V) \to \mathcal{H}_{\mathcal{S}}(V)$ as follows
\begin{equation}\label{eqn:CKME_x}
    \mu^{1}_{Y \mid X=x} = \mathbb{E}[k_{\mathcal{S}}(Y,\cdot)|X=x]=\int_{y \in \mathcal{X}(V)} k_{\mathcal{S}}(\cdot,y)\mathbb{P}_{Y \mid X=x}(dy) 
\end{equation}

\begin{equation}\label{eqn:CKME_X}
    \mu^{1}_{Y \mid X} = \mathbb{E}[k_{\mathcal{S}}(Y,\cdot)|X]=\int_{y \in \mathcal{X}(V)} k_{\mathcal{S}}(\cdot,y)\mathbb{P}_{Y \mid X}(dy)
\end{equation}

Note that whilst $\mu^{1}_{Y \mid X=x}$ is a single point in $\mathcal{H}_{\mathcal{S}}(V)$, the $1^{\text{st}}$ order conditional KME $\mu^{1}_{Y \mid X}$ describes a cloud of points on $\mathcal{H}_\mathcal{S}(V)$. Each point in this cloud is indexed by a path $x \in \mathcal{X}(V)$. Equivalently, $\mu^{1}_{Y \mid X}$ constitutes a $\mathbb{P}_X$-measurable, $\mathcal{H}_\mathcal{S}(V)$-valued random variable. 


These embeddings allow to extend the applications of conditional KMEs to the case where the random variables are (possibly multidimensional) stochastic processes. In particular one can directly obtain conditional independence criterions for stochastic processes (see \Cref{ssec:HSICstopro}), enabling to deploy standard kernel-based causal learning algorithms \cite{sun2007kernel}, as we demonstrate in \Cref{sec:applications}. Next we describe how in the case of stochastic processes, conditioning on filtrations is an important mathematical operation to model real-world time-evolving systems.

\subsection{Conditioning stochastic processes on their filtrations}\label{sec:cond_filtration}

Financial markets are examples of complex dynamical systems that evolve under the influence of randomness. An important objective for financial practitioners is to determine how actionable information available up to present could affect the future market trajectories. The task of \emph{conditioning on the past to describe the future} of a stochastic process $X$ can be formulated mathematically by conditioning $X$ on its filtration $\mathcal{F}_{X_t}$ for any time $t \in [0,T]$.

More precisely, consider the $1^{\text{st}}$ order KME of the conditional law $\mathbb{P}_{X \mid \mathcal{F}_{X_t}}$, which is defined as the following $\mathcal{F}_{X_t}$-measurable, $\mathcal{H}_{\mathcal{S}}(V)$-valued random variable
\begin{equation}\label{eqn:CSKME_F_t}
    \mu^{1}_{X \mid \mathcal{F}_{X_t}} = \mathbb{E}[k_{\mathcal{S}}(X,\cdot)|X_{[0,t]}] = \int_{x \in \mathcal{X}(V)} k_{\mathcal{S}}(\cdot,x)\mathbb{P}_{X \mid \mathcal{F}_{X_t}}(dx)
\end{equation}
where $X_{[0,t]}$ denotes the stochastic process $X$ restricted to the sub-interval $[0,t ] \subset [0,T]$. By varying the time index $t$, we can form the following ordered collection of $1^{\text{st}}$ order KMEs 
\begin{equation}\label{eqn:order-1-pred}
    \mu^{1}_{X \mid \mathcal{F}_X} = \left(\mu^{1}_{X \mid \mathcal{F}_{X_t}}\right)_{t \in [0,T]}
\end{equation}
that we term \emph{$1^{\text{st}}$ order predictive KME} of the process $X$. By construction, $\mu^{1}_{X \mid \mathcal{F}_X}$ describes a path taking its values in the space of $\mathcal{H}_{\mathcal{S}}(V)$-valued random variables, in other words it is itself a stochastic process \footnote{Because $\mathbb{P}_X = \mathbb{P}_{X \mid \mathcal{F}_{X_T}}$, all the information about the law of $X$ is contained in just the terminal point of the trajectory traced by $\mu^{1}_{X \mid \mathcal{F}_X}$ (\Cref{fig:main}).} (see \Cref{fig:main}). Hence, the law of $\mu^{1}_{X \mid \mathcal{F}_X}$ can itself be embedded via KMEs into a "higher-order RKHS" (see next section), making the full procedure iterable, as we shall discuss next.

We note that for each time $t$, the random variable $\mu^{1}_{X \mid \mathcal{F}_{X_t}}$ in \cref{eqn:CSKME_F_t} is the Bochner integral of 
$k_{\mathcal{S}}(\cdot,x)$ with respect to the probability measure $\mathbb{P}_{X \mid \mathcal{F}_{X_t}}$. Since we assumed that 
$V$ is a compact set, the path space $\mathcal{X}(V)$ is also compact. Hence, the function $x \mapsto k_{\mathcal{S}}(\cdot,x)$ is continuous, the set $K=\{k_\mathcal{S}(\cdot,x) : x \in \mathcal{X}(V)\}$ is compact as continuous image of a compact set, and therefore its Bochner integral $\mu^{1}_{X \mid \mathcal{F}_{X_t}}$ takes values in the closed convex hull of $K$, which is again a compact subset in the RKHS $\mathcal{H}_\mathcal{S}(V)$. Consequently the path  $t \mapsto \mu^{1}_{X \mid \mathcal{F}_{X_t}}$ belongs to a compact subset of $\mathcal{X}(\mathcal{H}_\mathcal{S}(V))$, which satisfies the assumptions introduced in Section 2.

\subsection{Second order kernel mean embedding and maximum mean discrepancy}\label{sec:second_order}

The \emph{$2^{\text{nd}}$ order KME} is the point in $\mathcal{H}_{\mathcal{S}}(\mathcal{H}_{\mathcal{S}}(V))$ defined as the KME of the $1^{\text{st}}$ order predictive KME
\begin{equation}
    \mu^2_X = \int_{x \in \mathcal{X}(\mathcal{H}_{\mathcal{S}}(V))} k_{\mathcal{S}}(\cdot,x)\mathbb{P}_{\mu^{1}_{X \mid \mathcal{F}_X}}(dx)
\end{equation}
The \emph{$2^{\text{nd}}$ order MMD} of $X,Y$ is the norm of the difference in $\mathcal{H}_{\mathcal{S}}(\mathcal{H}_{\mathcal{S}}(V))$ of their $2^{\text{nd}}$ order KMEs,
\begin{equation}\label{eq: MMD 1}
    \mathcal{D}_{\mathcal{S}}^2(X,Y) = \norm{\mu^2_X - \mu^2_Y}_{\mathcal{H}_{\mathcal{S}}(\mathcal{H}_{\mathcal{S}}(V))}
\end{equation}
The next theorem states that the $2^{\text{nd}}$ order MMD of two stochastic processes $X,Y$ is a stronger discrepancy measure than the $1^{\text{st}}$ order MMD. 

\begin{theorem}\label{thm:mmd_1_test}
Given two stochastic processes $X,Y$
\begin{align}
    \mathcal{D}^2_{\mathcal{S}}(X,Y)  = 0 
    \iff \mathbb{P}_{X\mid \mathcal{F}_{X}}  = \mathbb{P}_{Y \mid \mathcal{F}_Y} \quad
\end{align}
Furthermore
\begin{equation}
    \mathcal{D}^2_{\mathcal{S}}(X,Y) = 0 \implies \mathcal{D}^1_{\mathcal{S}}(X,Y) = 0
\end{equation}
but the converse is not generally true.
\end{theorem}

\begin{proof}
All proofs are given in \Cref{sec:proofs}.
\end{proof}

Next we make use of Thm. \ref{thm:mmd_1_test} in the context of two-sample hypothesis testing \cite{gretton2012kernel, chevyrev2018signature} for stochastic processes. In Sec. \ref{sec:applications} we will show by means of a numerical example that the $2^{\text{nd}}$ order MMD is able to capture filtration-related information otherwise ignored by the $1^{\text{st}}$ order MMD.

\subsection{A filtration-sensitive kernel two-sample test}\label{sec:sample_test}

Suppose we are given $m$ sample paths $\{x^i\}_{i=1}^m \sim X$ and $n$ sample paths $\{y^i\}_{i=1}^n \sim Y$. A classical two-sample test \cite{gretton2012kernel} for $X,Y$ tests a null-hypothesis 
\begin{equation}
    H_0 : \mathbb{P}_X = \mathbb{P}_Y \ \text{ against the alternative } \ H_A : \mathbb{P}_X \neq \mathbb{P}_Y
\end{equation}
The probability of falsely rejecting the null is called the \emph{type I error} (and similarly the probability of
falsely accepting the null is called the \emph{type II error}). If the type I error can be bounded from above by a constant $\alpha$, then we say that the test is of level $\alpha$. In \cite[Sec. 8]{chevyrev2018signature} it is shown that rejecting the null if $\widehat{\mathcal{D}}_\mathcal{S}^1(X,Y)^2 > c_\alpha$ (for some $c_\alpha$ that depends on $m,n$ and $\alpha$) gives a test of level $\alpha$, where $\widehat{\mathcal{D}}_\mathcal{S}^1(X,Y)$ denotes the classical unbiased estimator of the $1^{\text{st}}$ order MMD \cite{gretton2012kernel}. This choice of threshold is conservative and can be improved by using data-dependent bounds such as in permutation tests (we refer to the MMD testing literature for extra details \cite{gretton2012kernel, sriperumbudur2010hilbert, sejdinovic2013equivalence}).

However, as discussed in Sec. \ref{sec:second_order} comparing the laws $\mathbb{P}_X, \mathbb{P}_Y$ via the estimator above might be insufficient to capture filtration-related information about $X,Y$. To overcome this limitation we propose instead to test the null-hypothesis
\begin{equation}\label{eqn:hypo_}
    H_0 : \mathbb{P}_{X \mid \mathcal{F}_X} = \mathbb{P}_{Y \mid \mathcal{F}_Y} \ \text{ against the alternative } \ H_A : \mathbb{P}_{X \mid \mathcal{F}_X} \neq \mathbb{P}_{Y \mid \mathcal{F}_Y}
\end{equation}
Using \Cref{thm:mmd_1_test}, one can immediately construct a filtration-sensitive kernel two-sample test for (\ref{eqn:hypo_}) provided one can build an empirical estimator of the $2^{\text{nd}}$ order MMD $\mathcal{D}^2_{\mathcal{S}}(X,Y)$. In the rest of this section we explain how to obtain such an estimator and ultimately show its consistency. 

Assuming availability of $m$ sample paths $\{\widetilde{x}^i\}_{i=1}^{m}$ from the stochastic process $\mu^{1}_{X \mid \mathcal{F}_X}$ and $n$ sample paths $\{\widetilde{y}^i\}_{i=1}^{n}$ from $\mu^{1}_{Y \mid \mathcal{F}_Y}$, an estimator of the squared $2^{\text{nd}}$ order MMD is given by
\begin{equation*}
    \widehat{\mathcal{D}}^2_{\mathcal{S}}(X,Y)^2 = \frac{1}{m(m-1)} \sum_{\substack{i,j=1 \\ i\neq j}}^m k_{\mathcal{S}}(\widetilde{x}^i,\widetilde{x}^j) - \frac{2}{mn} \sum_{i,j=1}^{m,n} k_{\mathcal{S}}(\widetilde{x}^i,\widetilde{y}^j) + \frac{1}{n(n-1)} \sum_{\substack{i,j=1 \\ i\neq j}}^n k_{\mathcal{S}}(\widetilde{y}^i,\widetilde{y}^j)
\end{equation*}
Computing this estimator boils down to evaluating the  signature kernel $k_{\mathcal{S}}(\widetilde{x},\widetilde{y})$ on sample paths $\widetilde{x}\sim\mu^{1}_{X \mid \mathcal{F}_X}$ and $\widetilde{y}\sim\mu^{1}_{Y \mid \mathcal{F}_Y}$. By \Cref{thm:sig_PDE}, the signature kernel solves the following PDE
\begin{equation*}
    \frac{\partial^2 u_{\widetilde{x},\widetilde{y}}}{\partial s \partial t}  = \left(\big\langle \widetilde{x}_{s-\delta},\widetilde{y}_{t-\delta}\big\rangle_{\mathcal{H}_\mathcal{S}(V)} - \big\langle \widetilde{x}_{s-\delta},\widetilde{y}_{t}\big\rangle_{\mathcal{H}_\mathcal{S}(V)} - \big\langle \widetilde{x}_{s},\widetilde{y}_{t-\delta}\big\rangle_{\mathcal{H}_\mathcal{S}(V)} + \big\langle \widetilde{x}_{s},\widetilde{y}_{t}\big\rangle_{\mathcal{H}_\mathcal{S}(V)} \right) u_{\widetilde{x},\widetilde{y}}
\end{equation*}
where the two derivatives in \cref{eq:sig_PDE} have been approximated by finite difference with time increment $\delta$. It remains to explain how to estimate, for any $s,t \in [0,T]$, the inner product $\big\langle \widetilde{x}_{s},\widetilde{y}_{t}\big\rangle_{\mathcal{H}_\mathcal{S}(V)}$ from sample paths of $X$ and $Y$. This can be achieved using the formalism of \emph{cross-covariance operators} \cite{muandet2016kernel} as thoroughly explained in \Cref{sec:crosscovar}, which yields to the following approximation
\begin{equation}\label{eq:innerprodapprox}
    \big\langle \widetilde{x}_{s},\widetilde{y}_{t}\big\rangle_{\mathcal{H}_\mathcal{S}(V)} \approx {\mathbf{k}_s^x}^\top (\mathbf{K}_{s,s}^{x,x}+ m\lambda I_{m})^{-1}\mathbf{K}_{T,T}^{x,y}(\mathbf{K}_{t,t}^{y,y}+n\lambda I_{n})^{-1}\mathbf{k}_t^y
\end{equation}
where $\mathbf{k}_s^x \in \mathbb{R}^m, \mathbf{k}_t^y \in \mathbb{R}^n$ are the vectors\footnote{Here we use the notation $x_{[0,s]}$ to denote the restriction of the path $x$ to the sub-interval $[0,s] \subset [0,T]$.}
\begin{equation*}
     [\mathbf{k}_s^x]_i=k_{\mathcal{S}}(x^i_{[0,s]} ,x_{[0,s]}), \quad [\mathbf{k}_t^y]_i=k_{\mathcal{S}}(y^i_{[0,t]},y_{[0,t]})
\end{equation*}
and $\mathbf{K}_{s,s}^{x,x}\in\mathbb{R}^{m\times m}$, $\mathbf{K}_{T,T}^{x,y}\in\mathbb{R}^{m\times n}$, $\mathbf{K}_{t,t}^{y,y}\in\mathbb{R}^{n\times n}$ are the matrices 
\begin{equation*}
    [\mathbf{K}_{s,s}^{x,x}]_{i,j}=k_{\mathcal{S}}(x^i_{[0,s]},x^j_{[0,s]}), \quad [\mathbf{K}_{T,T}^{x,y}]_{i,j}=k_{\mathcal{S}}(x^i_{[0,T]},y^j_{[0,T]}), \quad [\mathbf{K}_{t,t}^{y,y}]_{i,j}=k_{\mathcal{S}}(y^i_{[0,t]},y^j_{[0,t]})
\end{equation*}
and where $I_m$ (resp. $I_n$) is the $m \times m$ (resp. $n \times n$) identity matrix. The corresponding algorithm and its complexity analysis are provided in \Cref{sec:algo}.  

The next theorem ensures that the estimator $\widehat{\mathcal{D}}^2_{\mathcal{S}}(X,Y)$ is consistent for the $2^{\text{nd}}$ order MMD.
\begin{theorem}
$\widehat{\mathcal{D}}^2_{\mathcal{S}}(X,Y)$ is a consistent estimator for the $2^{\text{nd}}$ order MMD, i.e.
\begin{equation}
    |\widehat{\mathcal{D}}^2_{\mathcal{S}}(X,Y) - \mathcal{D}^2_{\mathcal{S}}(X,Y)| \overset{p}{\to} 0 \quad \text{ as } \ m,n \to\infty
\end{equation}
with $\{x^i\}_{i=1}^m \sim X$, $\{y^i\}_{i=1}^n \sim Y$ and where convergence is in probability.
\end{theorem}

We now iterate the procedure presented so far to define higher order KMEs and MMDs.

\subsection{Higher order kernel mean embeddings and maximum mean discrepancies}\label{sec:higher_order}

One can iterate the procedure described in Sec. \ref{sec:second_order} and recursively define, for any $n \in \mathbb{N}_{>1}$, the \emph{$n^{\text{th}}$ order KME of $X$} as the following point in $\mathcal{H}^n_{\mathcal{S}}(V)$
\begin{equation}\label{eq: higher rank signature kernel mean embedding}
    \mu^n_X= \int_{x \in \mathcal{X}(\mathcal{H}_{\mathcal{S}}^{n-1}(V))} k_{\mathcal{S}}(\cdot,x)\mathbb{P}_{\mu^{n-1}_{X \mid \mathcal{F}_X}}(dx)
\end{equation}
where $\mu^{n-1}_{X \mid \mathcal{F}_X}$ is the $(n-1)^{\text{st}}$ predictive KME of $X$ and
\begin{equation}
    \mathcal{H}^n_{\mathcal{S}}(V) = \underbrace{\mathcal{H}_{\mathcal{S}}(\mathcal{H}_{\mathcal{S}}(\ldots\mathcal{H}_{\mathcal{S}}}_{n \text{ times}}(V)\ldots))
\end{equation}
The associated \emph{$n^{\text{th}}$ order MMD} between two processes $X,Y$ is then defined as the norm of the difference in $\mathcal{H}^n_{\mathcal{S}}(V)$ of the two $n^{\text{th}}$ order KMEs
\begin{equation}\label{eq: MMD n}
    \mathcal{D}_{\mathcal{S}}^n(X,Y) = \norm{\mu^n_X - \mu^n_Y}_{\mathcal{H}_{\mathcal{S}}^n(V)}
\end{equation}
The following result generalizes Thm. \ref{thm:mmd_1_test} in that it shows that the $n^{\text{th}}$ order MMD is a stronger (i.e. finer) discrepancy measure than all the $k^{\text{th}}$ order MMDs of lower order $1<k<n$.
\begin{theorem}
Given two stochastic processes $X,Y$
\begin{equation}
    \mathcal{D}^n_{\mathcal{S}}(X,Y) = 0 \implies \mathcal{D}^k_{\mathcal{S}}(X,Y) = 0 \quad \text{for any } \ 1<k<n
\end{equation}
but the converse is not generally true.
\end{theorem}
Other than hypothesis testing, another important application relying on the ability of distinguishing random variables is \emph{distribution regression} (DR) \cite{szabo2016learning}. In the next section we make use of the $n^{\text{th}}$ order MMD in the setting of DR on path-valued random variables presented in \cite{lemercier2021distribution} and propose a family of kernels on stochastic processes whose RKHSs contains richer classes of functions than the RKHS associated to the universal kernel proposed in \cite{lemercier2021distribution}.

We note that since $V$ is a Polish space (i.e., a separable, complete metric space) and the signature maps is continuous, in view of \cite[Lemma 4.33]{Christmann2008SVM} one can easily check that all RKHSs appearing in the present paper are separable Hilbert spaces by an induction argument and therefore all regular conditional distributions are well–defined.

\subsection{Higher order distribution regression}\label{sec:dr}

DR on stochastic processes describes the supervised learning problem where the input is a collection of sample paths and the output is a vector of scalars \cite{lemercier2021distribution}. Denote by $\mathcal{P}(\mathcal{X}(V))$ the set of stochastic processes with sample paths on $\mathcal{X}(V)$. Following the setup in \cite{lemercier2021distribution}, the goal is to learn a function $F : \mathcal{P}(\mathcal{X}(V)) \to \mathbb{R}$ from a training set of input-output pairs $\{(X_i, y_i)\}$ with $X_i \in \mathcal{P}(\mathcal{X}(V))$ and $y_i \in \mathbb{R}$, by means of a classical two-step procedure \cite{law2018bayesian, muandet2012learning, smola2007hilbert}. 

Firstly, a stochastic process $X \in \mathcal{P}(\mathcal{X}(V))$ is embedded into its KME $\mu_X^1 \in \mathcal{H}_\mathcal{S}(V)$ via the signature kernel $k_\mathcal{S}$. Secondly, another function $G:\mathcal{H}_\mathcal{S}(V) \to \mathbb{R}$ is learnt by solving the minimization $\argmin_{G \in \mathcal{H}_{\text{RBF}}}\sum_i\mathcal{L}(g(\mu_{X_i}^1),y_i)$, where $\mathcal{L}$ is a loss function, and $\mathcal{H}_{\text{RBF}}$ is the RKHS associated to the classical Gaussian kernel $k_{\text{RBF}}:\mathcal{H}_\mathcal{S}(V) \times \mathcal{H}_\mathcal{S}(V) \to \mathbb{R}$. This procedure materialises into a kernel on stochastic processes whose RKHS is shown to be dense in the space of functions $F: \mathcal{P}(\mathcal{X}(V)) \to \mathbb{R}$ that are continuous with respect to the weak topology \cite[Thm. 3.3]{lemercier2021distribution}. 

However, a class of approximators that is universal with respect to some topology is not guaranteed to well approximate functions that are discontinous with respect to that topology (but potentially continuous with respect to a finer topology). For example, financial practitioners are often interested in calibrating financial models to market data or pricing financial instruments from observations of market dynamics. These tasks can be formulated as DR problems on stochastic processes (see experiments in \Cref{sec:dr_finance}), but the resulting learnable functions are discontinuous with respect to the $1^{\text{st}}$ order MMD whilst being continuous with respect to the $2^{\text{nd}}$ order MMD \cite{Backhoff2019adapted}. This motivates the need to extend the kernel-based DR technique proposed in \cite{lemercier2021distribution} to situations where the target functions are not weakly continuous, which is what \Cref{thm:higher_rank_dr} addresses. A function $f:\mathbb{R} \to \mathbb{R}$ is called \emph{globally analytic with non-negative coefficients} if admits everywhere a Taylor expansion where all the coefficients are strictly positive, i.e. for any $x \in \mathbb{R}$ we have $f(x) = \sum_{i=0}^\infty a_i x^i$ with $a_i > 0$.

\begin{theorem}\label{thm:higher_rank_dr} 
Let $f: \mathbb{R} \to \mathbb{R}$ be a globally analytic function with non-negative coefficients. Define the family of kernels $K^n_{\mathcal{S}}:\mathcal{P}(\mathcal{X}(V))\times\mathcal{P}(\mathcal{X}(\mathbb{R}^d))\to\mathbb{R}$ as follows
\begin{equation}\label{eqn:univ_kernel}
K^n_{\mathcal{S}}(X,Y) = f(\mathcal{D}_\mathcal{S}^n(X,Y)), \quad n \in \mathbb{N}_{\geq 1}
\end{equation}
Then the RKHS associated to $K^n_{\mathcal{S}}$ is dense in the space of functions from $\mathcal{P}(\mathcal{X}(\mathbb{R}^d))$ to $\mathbb{R}$ which are continuous with respect to the $k^{\text{th}}$ order MMD for any $1<k\leq n$.
\end{theorem}

In \Cref{sec:applications} we will take $f(x) = \exp(-x^2/\sigma)$ with $\sigma>0$. This result marks the end of our analysis. Next we apply our theoretical results in the contexts of two-sample testing, DR and causal inference. 

\section{Applications}\label{sec:applications}
Here we demonstrate the practical advantage of using $2^{\text{nd}}$ order kernel mean embeddings, and evaluate the conditional kernel mean embedding for stochastic processes on a causal discovery task. Additional experimental details can
be found in \Cref{sec:experimentaldetails} and the code is available at \url{https://github.com/maudl3116/higherOrderKME}. 

\subsection{Hypothesis testing on filtrations}\label{sec:hypo}

We start by considering two processes $X^n$ and $X$ with transition probabilities depicted in \Cref{fig:example}. Although the laws $\mathbb{P}_n$ and $\mathbb{P}$  get arbitrarily close for large
$n$, their filtrations are very different. Indeed, the two processes have different information structures available
before time $t=1$. Indeed, for any $0<t\leq 1$, the trajectory of $X^n$ is deterministic, whilst the progression of $X$ remains random until $t=1$. 
\begin{figure}[h]
\begin{center}
	\begin{tikzpicture}[shorten >=1pt,draw=black!50,scale=0.62]
	\draw[black, fill=black, thick] (0,0) circle (2pt);
	\draw[black, fill=black, thick] (3,.2) circle (2pt);
	\draw[black, fill=black, thick] (3,-.2) circle (2pt);
	\draw[black, thick] (0,0) -- (3,.2);
	\node () at (1.5,.4) {\footnotesize$ p=0.5 $};
	\draw[black, thick] (0,0) -- (3,-.2);
	\node () at (1.5,-.5) {\footnotesize$ p=0.5 $};
	\draw [black, thick, decorate,decoration={brace,amplitude=4pt}] (3.15,0.2) -- (3.15,-.2) node [black,midway,xshift=9pt] {\footnotesize$\frac{2}{n}$};
	\draw[black, thick] (3,.2) -- (6,1);
	\node () at (4.5,.9) {\footnotesize$ p=1 $};
	\draw[black, thick] (3,-.2) -- (6,-1);
	\node () at (4.5,-1) {\footnotesize$ p=1 $};
	\draw[black, fill=black, thick] (6,1) circle (2pt);
	\draw[black, fill=black, thick] (6,-1) circle (2pt);
	
	\draw[black, fill=black, thick] (7,0) circle (2pt);
	\draw[black, fill=black, thick] (10,0) circle (2pt);
	\draw[black, thick] (7,0) -- (10,0);
	\node () at (8.5,.25) {\footnotesize$ p=1 $};
	\draw[black, thick] (10,0) -- (13,1);
	\node () at (11.5,.90) {\footnotesize$ p=0.5 $};
	\draw[black, thick] (10,0) -- (13,-1);
	\node () at (11.5,-1.) {\footnotesize$ p=0.5 $};
	\draw[black, fill=black, thick] (13,1) circle (2pt);
	\draw[black, fill=black, thick] (13,-1) circle (2pt);
	\end{tikzpicture}
\end{center}
	\caption{\small The supports of $\mathbb{P}_n$ (left) and $\mathbb{P}$ (right).}
	\label{fig:example}
\end{figure}
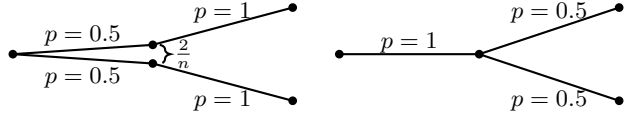
Being able to distinguish two such stochastic processes is crucial in quantitative finance: if $\mathbb{P}_n$ and $\mathbb{P}$ are the laws of two traded assets, $\mathbb{P}_n$ gives an arbitrage opportunity.
As shown in \Cref{fig:two_sample_test}, the $2^{\text{nd}}$ order MMD can distinguish these two processes with similar laws ($n=5\cdot10^5$) but different filtrations, while the $1^{\text{st}}$ order MMD fails to do so.
\begin{figure}[h]
    \centering
    \includegraphics[width=0.75\textwidth]{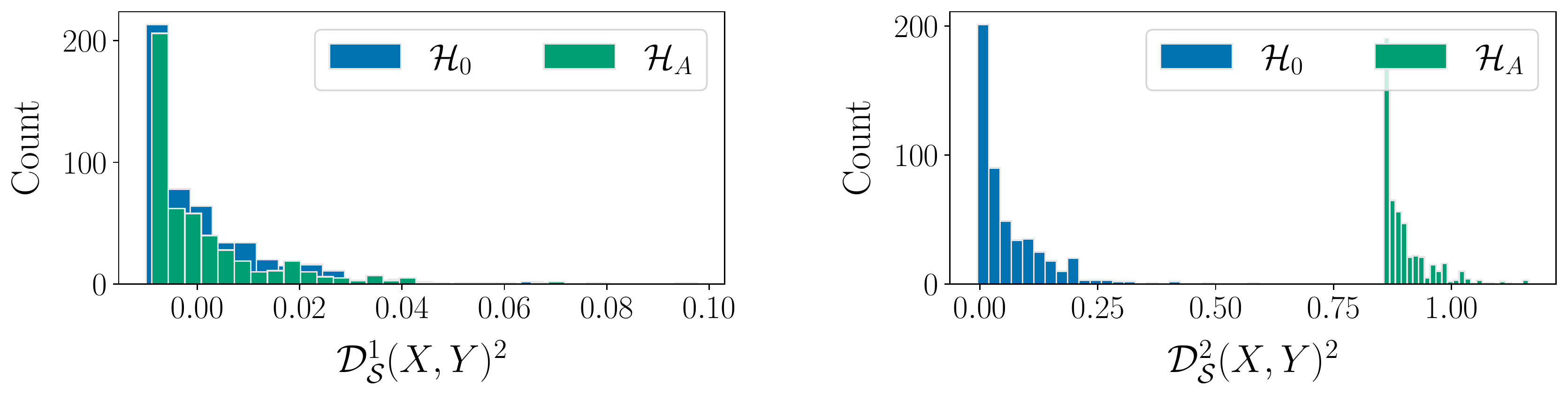}
    \caption{\small \textit{Left:} Empirical distribution of the $1^{\text{st}}$ order MMD. Under $\mathcal{H}_0$ the two measures are both equal to $\mathbb{P}$ and we use $500$ samples from each. Under $\mathcal{H}_A$ with $\mathbb{P}$ and $\mathbb{P}_n$ where $n=5\cdot10^5$, and we use $500$ samples. \textit{Right:} Same for the $2^{\text{nd}}$ order MMD. Histograms are obtained by computing $500$ independent instances of the MMD.}
    \label{fig:two_sample_test}
\end{figure}

\subsection{Applications of higher order distribution regression to quantitative finance}\label{sec:dr_finance}

In this section we use kernel Ridge regression and support vector machine (SVM) classification equipped with the kernel $K^2_\mathcal{S}$ from \Cref{thm:higher_rank_dr} to address two real-world problems arising in quantitative finance, notably the calibration of the \emph{rough Bergomi model} \cite{bayer2016pricing} and the pricing of \emph{American options} \cite{buehler2019deep}. We benchmark our filtration-sensitive kernel $K^2_\mathcal{S}$ against a range of kernels, including $K^1_\mathcal{S}$.



The rough Bergomi model is a rough volatility model \cite{gatheral2018volatility} satisfying the following stochastic dynamics 
\begin{equation}\label{eq:RoughBergomi}dS_t = \sqrt{V_t}S_tdW_t, \quad V_t = \int_0^t K(s,t)dZ_s, \quad Z_t = \rho W_t + \sqrt{1-\rho^2} W_t'
\end{equation}
where $W,W'$ are two independent Brownian motions and $K(s,t) = (t-s)^{h-0.5}$ where here we take $h=0.2$. The model in \cref{eq:RoughBergomi} is non-Markovian in the sense that the conditional law of $S \mid \mathcal{F}_{S_t}$ depends pathwise on the past history of the process. Of particular importance is the correct retrieval of the sign of the correlation parameter $\rho$ \cite{gassiat2019martingale}. We consider $50$ parameter values $\{\rho_i\}_{i=1}^{50}$ chosen uniformly at random from $[-1, 1]$. Each $\rho_i$ is regressed on a collection of $m=200$ sample trajectories. We use an SVM classifier endowed with different kernels (\Cref{table:finance}).
\begin{wraptable}{r}{0.6\textwidth}
\vspace{1\baselineskip}
\caption{\small Quantitative finance examples. Average performances with standard
errors in parenthesis.}
\begin{tabular}{l c c} 
 \toprule
   Kernel &   \thead{Rough Bergomi model\\calibration (Acc.)} & \thead{American option pricing\\ (MSE $\times 10^{-3}$)} \\ \midrule
    RBF & 87\% (5\%) & 1.07 (0.75)  \\ 
    Mat\'{e}rn &  87\% (3\%) & 2.75 (3.05)  \\
   $K^1_{\mathcal{S}}$  & 91\% (3\%) & 0.90 (0.34)  \\
   $K^2_{\mathcal{S}}$ & \textbf{93\%} (3\%) & \textbf{0.52} (0.07) \\ \bottomrule
\end{tabular}%
\vspace{-1\baselineskip}
\label{table:finance}
\end{wraptable} One of the most studied optimal stopping problems is the pricing of an American option with a non-negative payoff function $g : \mathbb{R}^d \to \mathbb{R}$. Stock prices are assumed to follow a $d$-dimensional stochastic process $X$. The price of the corresponding option is the solution of the optimal stopping problem $\sup_{\tau} \mathbb{E}[g(X_\tau) \mid X_0]$, where the supremum is taken over stopping times $\tau$. Despite significant advances, pricing American options remains one of the most computationally challenging problems in financial optimization, in particular when the underlying process $X$ is non-Markovian. This is the setting we consider, modelling stock prices as sample paths from fractional Brownian motion (fBm) \cite{duncan2000stochastic} with different Hurst exponents $h\in(0,1)$. True target prices are obtained via expensive Monte Carlo simulations \cite{longstaff2001valuing}. We consider $25$ values of $\{h_i\}_{i=1}^{25}$ sampled uniformly at random in $[0.2,0.8]$ and use $500$ samples from each fBm. As shown in \Cref{table:finance} our kernel $K^2_{\mathcal{S}}$ yields the best results on both tasks (rough Bergomi model calibration and American option pricing), systematically outperforming other classical kernels as well as the kernel $K^1_{\mathcal{S}}$ introduced in \cite{lemercier2021distribution}.

\subsection{Inferring causal graph for interacting bodies}\label{sec:causality}
Finally, we consider the task of recovering the causal relationships between interacting bodies solely from observations of their multidimensional trajectories. We employ the multi-body interaction simulator from \cite{li2020causal} in order to simulate an environment where $N$ balls are connected by invisible physical relations (e.g. a spring) and describe 2D trajectories (see Fig. 4a with $N=3$ and $2$ springs). At the beginning of a simulated episode, the initial positions of the balls are generated at random, and during the episode, the balls are subject to forces with random intensity and direction. By simulating $m$ episodes we end up with $m$ sample trajectories for each of the $N$ balls. We use the kPC algorithm \cite{sun2007kernel}---which relies on conditional independence testing--- with the signature kernel and evaluate its ability to recover whether any two balls are connected or not. We vary $m$ and $N$ and report the results in Figs. 4b and 4c. Each experiment is run $15$ times, $30\%$ of the runs are used to chose the hyperparameters, and the reported results have been obtained on the remaining runs. We note that for finite datasets conditional independence testing is hard without additional assumptions, as discussed in \cite{shah2020hardness, lundborg2021conditional}.

\begin{figure}[h]
     \centering
     \begin{subfigure}[b]{0.3\textwidth}
         \centering
          \includegraphics[scale=0.45,trim={2.5cm 1.8cm .5cm 2.cm},clip]{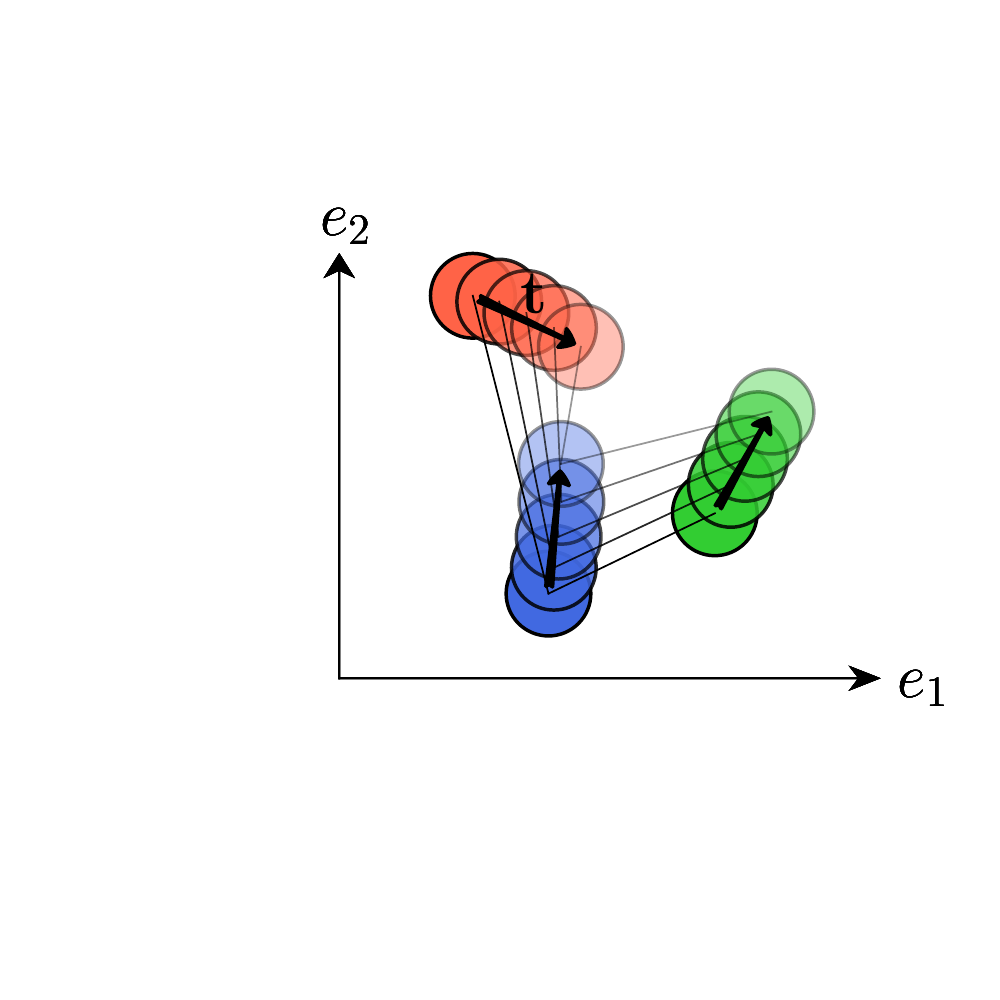}
         \caption{\small $3$ interacting balls describing trajectories in the 2D plane over time.}
         \label{fig:balls1}
     \end{subfigure}
     \hfill
     \begin{subfigure}[b]{0.3\textwidth}
         \centering
         \includegraphics[width=\textwidth]{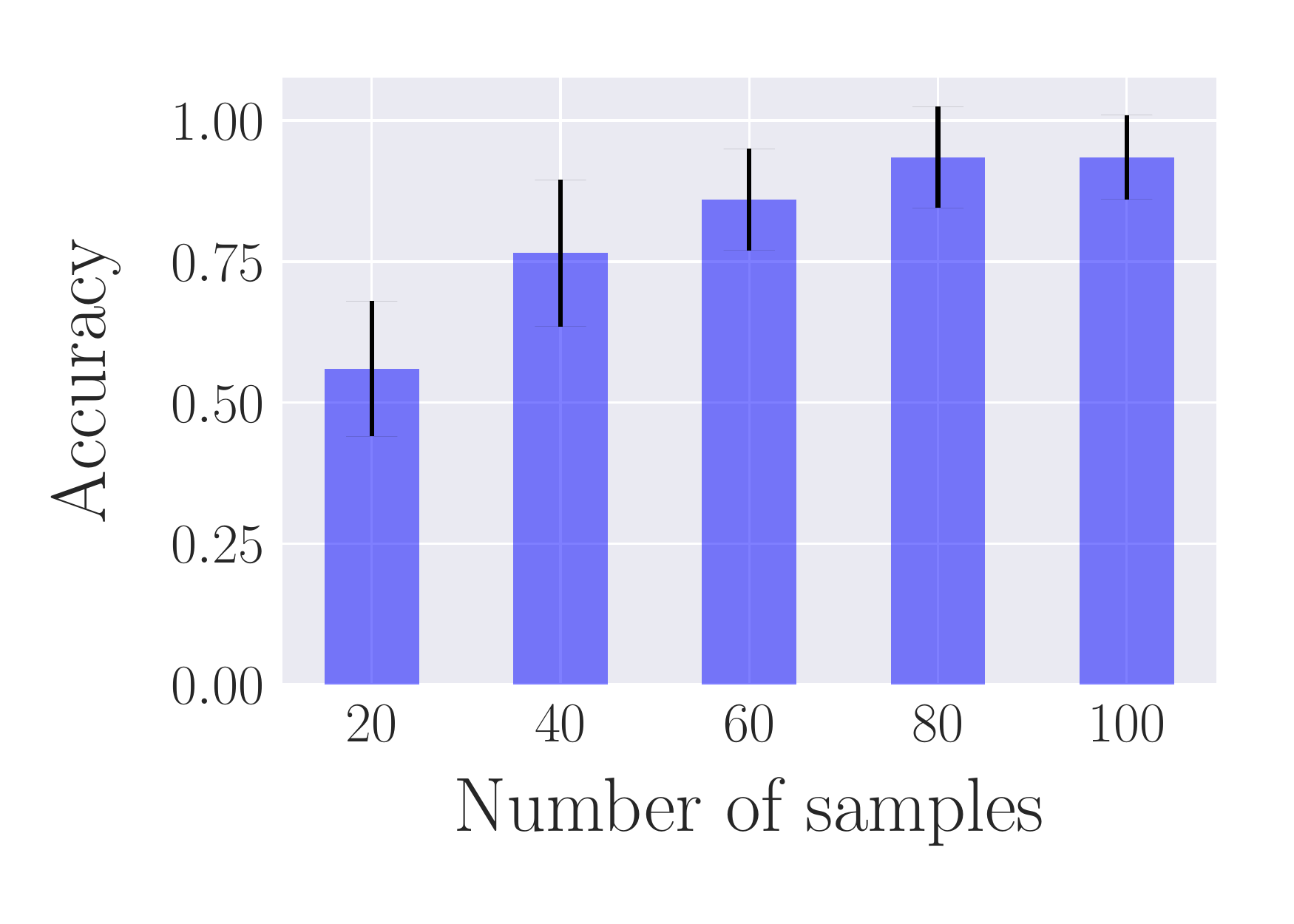}
         \caption{\small Accuracy on binary classification of edges with a varying number of sample episodes ($5$ balls)}
         \label{fig:sample}
     \end{subfigure}
     \hfill
    \begin{subfigure}[b]{0.3\textwidth}
         \centering
         \includegraphics[width=\textwidth]{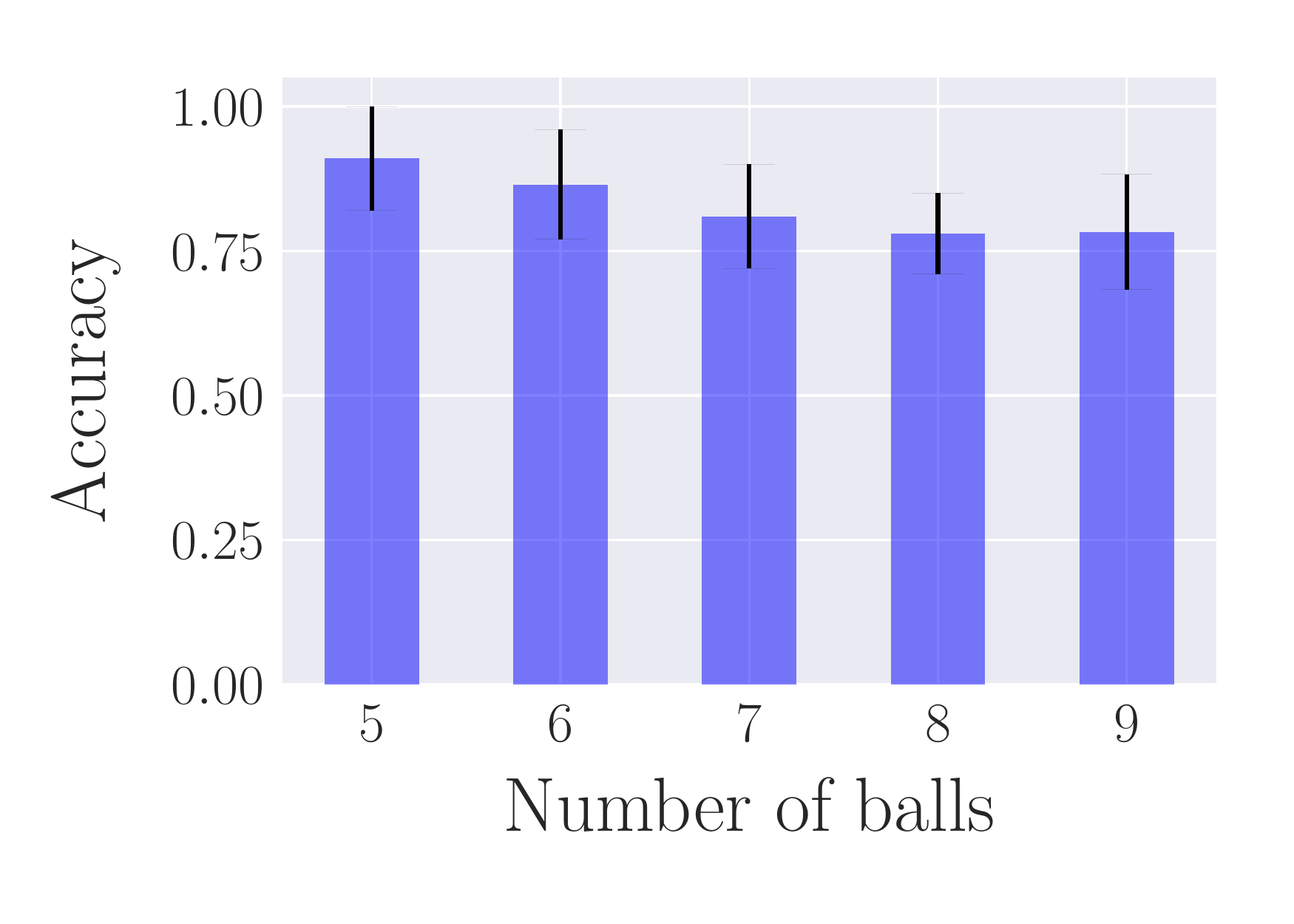}
        \caption{\small Accuracy on binary classification of edges with a varying number of balls ($100$ samples)}
         \label{fig:balls2}
     \end{subfigure}
        \label{fig:causal_discovery}
\end{figure}

\section{Conclusion}

In this paper, we introduced a family of higher order KMEs by conditioning a stochastic process on its filtration, generalizing the classical notion of KME. We derived an empirical estimator for the $2^{\text{nd}}$ order MMD and proved its consistency. We then proposed a filtration-sensitive kernel two-sample test and showed with simulations its ability to capture information that gets missed by the standard MMD test. In addition, we constructed a family of universal kernels on stochastic processes that allows to solve real-world calibration and optimal stopping problems in quantitative finance via Ridge regression. Finally, we designed a causal-discovery algorithm using conditional independence tests to recover the causal graph of structural dependencies among interacting bodies solely from observations of their multidimensional trajectories.

\section{Future work}

Regarding the choice of kernel hyperparameters, in the setting of two-sample tests, we can use various hyperparameter selection methods which have been proposed in the kernel literature, including approaches aiming at maximizing the test power using the signal-to-noise-ratio as an objective \cite{gretton2012optimal, sutherland2017, liu2020learning}. In the distribution regression setting, we made use of a cross validation approach. Developing hyperparameter tuning methodologies for higher order KMEs is an interesting future work direction and we will note that  \cite{flaxman2016bayesian, hsu2018hyperparameter} are certainly a good starting point for such an investigation.

Higher order KMEs have the potential to be used beyond two-sample tests and distribution regression. For example [6] recently investigated the use of the $1^{\text{st}}$ order MMD to derive an approximate Bayesian computation (ABC) algorithm for irregular time series. Another idea that is currently being investigated is using higher order MMDs as discriminators in autoregressive generative models for time series, where conditioning the future trajectories on past observations is key.

We conclude with a theoretical remark. All paths considered in the present paper are piecewise linear. Consequently, all sample paths from higher order predictive KMEs are also piecewise linear and their KMEs are well defined. Such a nice property will not hold anymore if one considered more generic continuous sample paths, because such regularity of sample paths from the corresponding higher order predictive KMEs might break as noted in \cite[Remark 1]{bonnier2020adapted}. The study of how the regularity changes by taking higher order kernel mean embeddings is an interesting direction for future work.

\begin{ack}
ML and CS were respectively supported by the EPSRC grants EP/L016710/1 and EP/R513295/1. TD acknowledges support from EPSRC (EP/T004134/1), UKRI Turing AI Fellowship (EP/V02678X/1), and the Lloyd’s Register Foundation programme on Data Centric Engineering through the London Air Quality project. ML, CS and TL were supported by the Alan Turing Institute under the EPSRC grant EP/N510129/1 and DataSig under the grant EP/S026347/1.
\end{ack}

\bibliographystyle{unsrt}
\bibliography{references}

\appendix

\newpage

\section*{\centering \LARGE Appendix}

This Appendix is organised as follows: A) using the formalism of cross-covariance operators we define an Hilbert-Schmidt conditional independence criterion for stochastic processes, and provide further details on the construction of the estimator for the $2^{\text{nd}}$ order MMD; B) we outline algorithms and their complexities for computing higher order MMDs; C) we provide further experimental details; D) we prove the theorems from the main paper.  

\section{Cross-covariance operators}\label{sec:crosscovar}

Covariance and cross-covariance
operators on RKHSs are important concepts for modern applications of conditional KMEs \cite{fukumizu2004dimensionality,song2009hilbert}. In this section we will use this formalism (adapted to the case of path-valued random variables) to firstly derive a criterion for conditional independence of stochastic processes and secondly provide more details on the derivation of our estimator of the $2^{\text{nd}}$ order MMD.

Let $X,Y \in \mathcal{P}(\mathcal{X}(V))$ be two stochastic processes and their joint process $(X,Y) \in \mathcal{P}(\mathcal{X}(V \oplus V))$. Define the \emph{cross-covariance operator} $\mathcal{C}_{Y,X}$ as the following point in the tensor product of RKHSs $\mathcal{H}_{\mathcal{S}}(V) \otimes \mathcal{H}_{\mathcal{S}}(V)$
\begin{equation}
    \mathcal{C}_{Y,X} = \mathbb{E}_{(X,Y)}[k_{\mathcal{S}}(\cdot,Y) \otimes k_{\mathcal{S}}(X,\cdot)]
\end{equation}
or equivalently as the \emph{Hilbert-Schmidt operator} $\mathcal{C}_{Y,X} : \mathcal{H}_{\mathcal{S}}(V) \to \mathcal{H}_{\mathcal{S}}(V)$ defined for any function $f \in \mathcal{H}_{\mathcal{S}}(V)$ as follows
\begin{equation}
    \mathcal{C}_{Y,X}(f)(\cdot) = \int_{(x,y) \in \mathcal{X}(V \oplus V)} k_{\mathcal{S}}(\cdot,y) f(x) \mathbb{P}_{(X,Y)}(d(x,y))
\end{equation}
The equivalence between tensor product of RKHSs and Hilbert-Schmidt operators is given by the isomorphism $\Phi: \mathcal{H}_{\mathcal{S}}(V) \otimes \mathcal{H}_{\mathcal{S}}(V) \to \text{HS}\left(\mathcal{H}_{\mathcal{S}}(V),\mathcal{H}_{\mathcal{S}}(V)\right)$ defined as follows
\begin{equation}
    \Phi_{\mathcal{S}}\left(\sum_{\mathbf{k},\mathbf{k}'} \alpha_{\mathbf{k},\mathbf{k}'} \mathcal{S}^{(\mathbf{k})} \otimes \mathcal{S}^{(\mathbf{k}')}\right) = \sum_{\mathbf{k},\mathbf{k}'} \alpha_{\mathbf{k},\mathbf{k}'} \left\langle \cdot, \mathcal{S}^{(\mathbf{k})} \right\rangle_{\mathcal{H}_{\mathcal{S}}(V)}\mathcal{S}^{(\mathbf{k}')}
\end{equation}
where $\mathcal{S}^{(\mathbf{k})}$ denotes the $\mathbf{k}^{\text{th}}$ element of an orthogonal basis of $\mathcal{H}_\mathcal{S}(V)$ and  $\text{HS}\left(\mathcal{H}_{\mathcal{S}}(V),\mathcal{H}_{\mathcal{S}}(V)\right)$ is the space of Hilbert-Schmidt operators from $\mathcal{H}_{\mathcal{S}}(V)$ to itself. An example of such basis is given by the \emph{signature basis} defined for any path $x \in \mathcal{X}(V)$ and any coordinate  $\mathbf{k}=(k_1,...,k_j)$ as  
\begin{equation}
    \mathcal{S}^{\mathbf{k}}(x) = \underset{0<s_1<\ldots<s_j<T}{\int \ldots \int} d x^{(k_1)}_{s_1} \ldots dx^{(k_j)}_{s_j}
\end{equation}
The centered version $\tilde{\mathcal{C}}_{Y,X} \in \mathcal{H}_{\mathcal{S}}(V) \otimes \mathcal{H}_{\mathcal{S}}(V)$ of the operator $\mathcal{C}_{Y,X}$ is defined as 
\begin{equation}
    \tilde{\mathcal{C}}_{Y,X} = \mathcal{C}_{Y,X} - \mu^1_X \otimes \mu^1_Y
\end{equation}
Similarly let $\mathcal{C}_{X,X} \in \mathcal{H}_{\mathcal{S}}(V) \otimes \mathcal{H}_{\mathcal{S}}(V)$ be the following covariance operator
\begin{equation}
    \mathcal{C}_{X,X} = \mathbb{E}_X[k_{\mathcal{S}}(X,\cdot)\otimes k_{\mathcal{S}}(X,\cdot)]
\end{equation}
or equivalently $\mathcal{C}_{X,X} \in \text{HS}\left(\mathcal{H}_{\mathcal{S}}(V),\mathcal{H}_{\mathcal{S}}(V)\right)$
\begin{equation}
    \mathcal{C}_{X,X}(f)(\cdot) = \int_{x \in \mathcal{X}(V)} k_{\mathcal{S}}(\cdot,x) f(x) \mathbb{P}_X(dx)
\end{equation}

Under the assumption that for any $f\in\mathcal{H}_{\mathcal{S}}(V)$ the function $x\mapsto \mu^1_{f(Y) \mid X=x}$ from $\mathcal{X}(V)$ to $\mathbb{R}$ is in $\mathcal{H}_{\mathcal{S}}(V)$, the authors in \cite{song2009hilbert,fukumizu2004dimensionality} showed that 
\begin{equation}\label{eq:ckme_op}
     \mu^1_{Y \mid X} = \mathcal{C}_{Y,X}\mathcal{C}_{X,X}^{-1}
\end{equation}
However, this assumption might not hold in general \cite{song2009hilbert,fukumizu2004dimensionality}. This technical issue can be circumvented by resorting to a regularized version of \cref{eq:ckme_op}: this yields to

\begin{equation}\label{eq:ckme_op_reg}
    \mu^1_{Y \mid X} \approx \mathcal{C}_{Y,X}(\mathcal{C}_{X,X}+\lambda I_{\mathcal{H}_{\mathcal{S}}(V)})^{-1}, \quad \lambda > 0
\end{equation}

where $I_{\mathcal{H}_{\mathcal{S}}(V)}$ is the identity map from $\mathcal{H}_{\mathcal{S}}(V)$ to itself. Under some mild conditions the empirical estimator of \cref{eq:ckme_op_reg} is equal to the empirical estimator of \cref{eq:ckme_op} \cite[Thm. 8]{fukumizu2013kernel}. In particular, one has 

\begin{equation}
    \hat{\mu}^1_{Y|X=x} =  {\mathbf{k}^{x}}^\top(\mathbf{K}^{x,x}+m\lambda I_m)^{-1}\mathbf{k}^{y}(\cdot)
\end{equation}

based on sample paths $\{(x^i,y^i)\}_{i=1}^{m}$ from the joint $(X,Y)$ where $\mathbf{k}^{x}$, $\mathbf{K}^{x,x}$ and $\mathbf{k}^{y}(\cdot)$ are such that

\begin{equation*}
    [\mathbf{k}^x]_i=k_{\mathcal{S}}(x^i,x) \quad [\mathbf{K}^{x,x}]_{i,j}=k_{\mathcal{S}}(x^i,x^j) \quad \mathbf{k}^{y}(\cdot)={[k_{\mathcal{S}}(y_1,\cdot),\ldots,k_{\mathcal{S}}(y_m,\cdot)]}^\top
\end{equation*}

Cross-covariance operators have been used to define kernel-based measures of conditional dependence, as we shall discuss in the next section.  
\subsection{Hilbert-Schmidt Conditional Independence Criterion for stochastic processes}\label{ssec:HSICstopro}
Multiple measures of conditional dependence have been proposed in the literature \cite{fukumizu2007kernel,sun2007kernel,tillman2009nonlinear}. In this section, we follow  \cite{tillman2009nonlinear} to define a nonparametric conditional dependence measure for stochastic processes, based on the \textit{conditional cross-covariance operator} $\tilde{\mathcal{C}}_{Y,X|Z}:\mathcal{H}_{\mathcal{S}}(V)\to\mathcal{H}_{\mathcal{S}}(V)$,
\begin{equation}
    \tilde{\mathcal{C}}_{Y,X|Z}= \tilde{\mathcal{C}}_{Y,X} - \tilde{\mathcal{C}}_{Y,Z}\tilde{\mathcal{C}}_{Z,Z}^{-1}\tilde{\mathcal{C}}_{Z,X}
\end{equation}
The squared Hilbert-Schmidt norm $H_{YX|Z}:=||\tilde{\mathcal{C}}_{(Y,Z),X|Z}||_{\text{HS}}^2$ can be used as measure of conditional dependence of stochastic processes. Since the signature kernel $k_{\mathcal{S}}$ is characteristic, it follows that $X \indep Y \mid Z \iff H_{YX|Z}=0$ \cite{tillman2009nonlinear}. 

Given $m$ sample paths $\{(x_i,y_i,z_i)\}_{i=1}^{m}$ from the joint distribution of $(X,Y,Z)$, let $K^x,K^y$ and $K^z$ be the Gram matrices with entries, 
\begin{equation*}
[K^{x}]_{i,j}=k_{\mathcal{S}}(x_i,x_j) \quad [K^{y}]_{i,j}=k_{\mathcal{S}}(y_i,y_j) \quad [K^{z}]_{i,j}=k_{\mathcal{S}}(z_i,z_j)
\end{equation*}
An empirical estimator of the kernel conditional dependence measure $H_{YX|Z}$ is then given by, 
\begin{equation*}
 \resizebox{\textwidth}{!}{%
    $\widehat{H}_{YX|Z}=\frac{1}{m^2}\left\{\mathrm{tr}(\tilde{K}^{x}\tilde{K}^{y}) -2\mathrm{tr}(\tilde{K}^{x}\tilde{K}^{z}(\tilde{K}^{z}_{\epsilon})^{-2}\tilde{K}^{z}\tilde{K}^{y}) + \mathrm{tr}(\tilde{K}^{x}\tilde{K}^{z}(\tilde{K}^{z}_{\epsilon})^{-2}\tilde{K}^{z}\tilde{K}^{y}\tilde{K}^{z}(\tilde{K}^{z}_{\epsilon})^{-2}\tilde{K}^{z}) \right\}$%
}
\end{equation*}
where $\tilde{K}^{z}_{\epsilon}=\tilde{K}^{z}+\epsilon I_m$ and $\tilde{K}^{x},\tilde{K}^{y}$,$\tilde{K}^{z}$ are the centered versions of the matrices $K^x,K^y$ and $K^z$,
\begin{equation*}
    \tilde{K}^{x}=HK^{x}H \quad  \tilde{K}^{y}=HK^{y}H \quad \tilde{K}^{z}=HK^{z}H \quad 
\end{equation*} 
with $H=I_m-{m}^{-1}\mathbf{1}_m$ and $\bold{1}_m$ the $m\times m$ matrix with all entries set to $1$. This estimator can be used as a test statistic for testing whether $X$ and $Y$ are independent given $Z$. However, it is not known how to analytically compute the null distribution of the test statistic, and permutation tests are typically used. In \Cref{sec:causality} we use this measure of conditional dependence as part of the kPC algorithm to infer causal relationships between multidimensional stochastic processes. We provide more details in \Cref{sec:experimentaldetails}.

\subsection{Construction of the estimator for the second order MMD $\mathcal{D}_\mathcal{S}^2$}

As discussed in the main paper, the estimation of the $2^{\text{nd}}$ order MMD, required the ability to compute inner products of the form $\langle\tilde{x}_s,\tilde{y}_t\rangle$ in $\mathcal{H}_{\mathcal{S}}(V)$. Here, we provide more details on the approximation that we have used in \cref{eq:innerprodapprox}, also restated below,
\begin{equation*}
    \big\langle \widetilde{x}_{s},\widetilde{y}_{t}\big\rangle_{\mathcal{H}_\mathcal{S}(V)} \approx {\mathbf{k}_s^x}^\top (\mathbf{K}_{s,s}^{x,x}+ m\lambda I_{m})^{-1}\mathbf{K}_{T,T}^{x,y}(\mathbf{K}_{t,t}^{y,y}+n\lambda I_{n})^{-1}\mathbf{k}_t^y
\end{equation*}
where $\widetilde{x}$ and $\widetilde{y}$ are sample paths from the processes $\mu^1_{X|\mathcal{F}_X}$ and $\mu^1_{Y|\mathcal{F}_Y}$ respectively. In particular,
\begin{equation*}
\widetilde{x}_s = \mu^1_{X|x_{[0,s]}} \quad\text{and}\quad \widetilde{y}_t = \mu^1_{Y|y_{[0,t]}}
\end{equation*}
As discussed at the beginning of this section, their empirical estimators are constructed from $m$ samples $\{x^i\}_{i=1}^{m}$ from $X$ and $n$ samples $\{y^j\}_{j=1}^{n}$ from $Y$ respectively 
\begin{equation*}\label{eqn:estimator_x}
    \widetilde{x}_s\approx  {\mathbf{k}^x_s}^\top(\mathbf{K}^{x,x}_{s,s}+m\lambda I_{m})^{-1}\mathbf{k}^{x}(\cdot) \quad\text{and}\quad \widetilde{y}_t\approx  {\mathbf{k}^y_t}^\top(\mathbf{K}^{y,y}_{t,t}+n\lambda I_{n})^{-1}\mathbf{k}^{y}(\cdot)
\end{equation*}
where $\mathbf{k}_s^{x}$, $\mathbf{K}^{x,x}_{s,s}$, $\mathbf{k}^{x}(\cdot)$ and $\mathbf{k}_t^{y}$, $\mathbf{K}^{y,y}_{t,t}$ and $\mathbf{k}^{y}(\cdot)$ are defined by,
\begin{align*}
    &[\mathbf{k}^{x}_s]_i=k_{\mathcal{S}}(x^i_{[0,s]},x^{}_{[0,s]}) \quad [\mathbf{K}^{x,x}_{s,s}]_{i,j}=k_{\mathcal{S}}(x^i_{[0,s]},x^j_{[0,s]}) \quad \mathbf{k}^{x}(\cdot)={[k_{\mathcal{S}}(x_1,\cdot),\ldots,k_{\mathcal{S}}(x_m,\cdot)]}^\top \\
   & [\mathbf{k}^{y}_t]_i=k_{\mathcal{S}}(y^i_{[0,t]},y^{}_{[0,t]}) \quad [\mathbf{K}^{y,y}_{t,t}]_{i,j}=k_{\mathcal{S}}(y^i_{[0,t]},y^j_{[0,t]}) \quad \mathbf{k}^{y}(\cdot)={[k_{\mathcal{S}}(y_1,\cdot),\ldots,k_{\mathcal{S}}(y_n,\cdot)]}^\top
\end{align*}
Alternatively, we can write $\widetilde{x}_s\approx\sum_{i=1}^{m}\alpha_i k_{\mathcal{S}}(x^i,\cdot)$ with $\boldsymbol{\alpha}=(\mathbf{K}^{x,x}_{s,s}+m\lambda I_{m})^{-1}\mathbf{k}^{x}_s$. Similarly we have $\widetilde{y}_t\approx\sum_{j=1}^{n}\beta_j k_{\mathcal{S}}(\cdot,y^j)$ with $\boldsymbol{\beta}=(\mathbf{K}^{y,y}_{t,t}+n\lambda I_{n})^{-1}\mathbf{k}^{y}_t$. Therefore, the inner product between $\widetilde{x}_s$ and $\widetilde{y}_t$ can be approximated as follows,
\begin{align*}
    \left\langle\widetilde{x}_s,\widetilde{y}_t\right\rangle_{\mathcal{H}_{\mathcal{S}}(V)} &\approx \sum_{i=1}^{m}\sum_{j=1}^{n}\alpha_i\beta_j k_{\mathcal{S}}(x^i,y^j)\\
    & = \boldsymbol{\alpha}^\top\mathbf{K}^{x,y}_{T,T}\boldsymbol{\beta} \\
    & ={\mathbf{k}^{x}_s}^\top(\mathbf{K}^{x,x}_{s,s}+m\lambda I_{m})^{-1}\mathbf{K}^{x,y}_{T,T}(\mathbf{K}^{y,y}_{t,t}+n\lambda I_{n})^{-1}\mathbf{k}^{y}_t
\end{align*}

where $\mathbf{K}^{x,y}_{T,T}\in\mathbb{R}^{m\times n}$ with $[\mathbf{K}^{x,y}_{T,T}]_{i,j}=k_{\mathcal{S}}(x^i,y^j)$. Next we outline the algorithm to compute $\hat{\mathcal{D}}^2_{\mathcal{S}}$. 

\section{Algorithms}\label{sec:algo}
In this section, we provide algorithms to compute the empirical estimator $\hat{\mathcal{D}}_{\mathcal{S}}^{k}$ for the $k^{\text{th}}$ order MMD, which rely on the ability to evaluate the signature kernel $k_{\mathcal{S}}(x,y)$ where $x$ and $y$ are two paths taking their values in the Hilbert space $\mathcal{H}^{k-1}(V)$. Following \cite[Sec. 3.1.]{cass2020computing} we use an explicit finite difference scheme to approximate the PDE solution $u_{x,y}$ on a grid $\mathcal{P}$ of size $P\times Q$, 
\begin{equation*}
    \mathcal{P}=\{0=s_1<s_2<\ldots<s_P=T\}\times\{0=t_1<t_2<\ldots<t_Q=T\}
\end{equation*}
Writing $u_{x,y}(s_i,t_j) = u_{i,j}$ to make the notation more concise, we use an update rule of the form,
\begin{equation*}
    u_{i+1,j+1} = f(u_{i,j+1},u_{i+1,j},u_{i,j},M_{i,j}), \quad M_{i,j}=\langle x_{s_{i+1}}-x_{s_i},y_{t_{j+1}}-y_{t_j}\rangle_{\mathcal{H}^{k-1}_{\mathcal{S}}(V)}
\end{equation*}
Hence, computing $k_{\mathcal{S}}(x,y)$ consists in forming the $(P-1)\times (Q-1)$ matrix $M$ such that, 
\begin{equation*}
    [M]_{i,j} = \langle x_{s_{i+1}}-x_{s_i},y_{t_{j+1}}-y_{t_j}\rangle_{\mathcal{H}^{k-1}_{\mathcal{S}}(V)}
\end{equation*}
and iteratively applying the update rule as outlined in \Cref{alg:pdesolve}. Besides, in \Cref{alg:pdesolve} we distinguish the case where the solution $u$ on the entire grid is returned, and the case where only the solution at the final times $(s_P,t_Q)=(T,T)$ is returned, which corresponds to the value of the kernel $k_{\mathcal{S}}(x,y)$. The runtime complexity to solve one PDE is $\mathcal{O}(PQ)$. We make use of parallelization strategy to drastically speed-up the PDE solver on CUDA-enabled GPUs. 

\subsection{Algorithm for the $1^{\text{st}}$ order MMD}
In this section we provide the algorithm to compute an empirical estimator of the $1^{\text{st}}$ order MMD. This way, we introduce subroutines (\Cref{alg:pdesolve} and \Cref{alg:firstordergram}) for the estimator of the $2^{\text{nd}}$ order MMD. We assume that $m=n$ and $P=Q$ to simplify the final runtime complexities of the algorithms. 

Let $\{x^i\}_{i=1}^{m}\sim X$ and  $\{y^j\}_{j=1}^{n}\sim Y$. An unbiased estimator of the $1^{\text{st}}$ order MMD \cite{gretton2012kernel} reads as,
\begin{equation*}
    \widehat{\mathcal{D}}^1_{\mathcal{S}}(X,Y) = \frac{1}{m(m-1)} \sum_{\substack{i,j=1 \\ i\neq j}}^m k_{\mathcal{S}}(x^i,x^j) - \frac{2}{mn} \sum_{i,j=1}^{m,n} k_{\mathcal{S}}(x^i,y^j) + \frac{1}{n(n-1)} \sum_{\substack{i,j=1 \\ i\neq j}}^n k_{\mathcal{S}}(y^i,y^j)
\end{equation*}
 Hence, in order to compute this estimator, we need to form the following three Gram matrices  $\mathbf{G}^{1}_{X,X}\in\mathbb{R}^{m\times m}$, $\mathbf{G}^{1}_{X,Y}\in\mathbb{R}^{m\times n}$ and $\mathbf{G}^{1}_{Y,Y}\in\mathbb{R}^{n\times n}$ such that, 
\begin{equation*}
    [\mathbf{G}^{1}_{X,X}]_{i,j}=k_{\mathcal{S}}(x^i,x^j) \quad   [\mathbf{G}^{1}_{X,Y}]_{i,j}=k_{\mathcal{S}}(x^i,y^j) \quad   [\mathbf{G}^{1}_{Y,Y}]_{i,j}=k_{\mathcal{S}}(y^i,y^j)
\end{equation*}
As explained at the begining of this section (and outlined in \Cref{alg:firstordergram}), this consists in two steps. Taking $\mathbf{G}^{1}_{X,Y}$ for example, first one forms $m\times n$ matrices of size $(P-1)\times (Q-1)$ each of the form, 
\begin{equation*}
    [M]_{p,q} = \langle x^i_{s_{p+1}}-x^i_{s_p},y^j_{t_{q+1}}-y^j_{t_q}\rangle_{V}
\end{equation*}
and then one solves $m\times n$ PDEs with \Cref{alg:pdesolve}. The full procedure is summarized in \Cref{alg:firstordermmd}, which has time complexity $\mathcal{O}(dm^2P^2)$ where $d$ is the number of coordinates of the paths $x$ and $y$. 

\begin{algorithm}
\caption{$\mathsf{PDESolve}$ \hfill $\mathcal{O}(P^2)$}
\label{alg:pdesolve}
\begin{algorithmic}[1]
\State {\bfseries Input:} matrix $M\in\mathbb{R}^{P\times Q}$, $\mathsf{full}\in\{\mathsf{True},\mathsf{False}\}$
\State {\bfseries Output:} full solution $u\in\mathbb{R}^{P\times Q}$ with $u[p,q]=k_{\mathcal{S}}(x_{[0,s_p]},y_{[0,t_q]})$ or $u[-1,-1]=k_{\mathcal{S}}(x,y)$
\vspace{5pt}
\State{$u[1,:]\leftarrow 1$}
\State{$u[:,1]\leftarrow 1$}
\For{$p$ from $1$ to $P-1$}
\For{$q$ from $1$ to $Q-1$}
\State $u[p+1,q+1]\leftarrow f (u[p,q+1], u[p+1,q], u[p,q], M[p,q] )$
\EndFor
\EndFor
\State {\bfseries if} $\mathsf{full}$ {\bfseries then return} $u$ {\bfseries else}  {\bfseries return} $u[-1,-1]$
\end{algorithmic}
\end{algorithm}

\begin{algorithm}
\caption{$\mathsf{FirstOrderGram}$ \hfill $\mathcal{O}(dm^2P^2)$}
\label{alg:firstordergram}
\begin{algorithmic}[1]
\State {\bfseries Input:} sample paths $\{x^i\}_{i=1}^{m}\sim X$ and $\{y^j\}_{j=1}^{n}\sim Y$, $\mathsf{full}\in\{\text{True},\text{False}\}$
\State {\bfseries Output:} $G\in\mathbb{R}^{m\times n\times P\times Q}$ where $G[i,j,p,q]=k_{\mathcal{S}}(x^i_{[0,s_p]},y^j_{[0,t_q]})$ or $G[:,:,-1,-1]$
\vspace{5pt}
\State $M[i,j,p,q] \leftarrow \langle x^i_{s_p}, y^j_{t_q}\rangle\quad \forall i\in\{1,\ldots,m\},~ j\in\{1,\ldots,n\}, p\in\{1,\ldots,P\},~ q\in\{1,\ldots,Q\}$ 
\State $M\leftarrow M[:,:,1{:},1{:}] + M[:,:,{:}\!-\!1,{:}\!-\!1] - M[:,:,1{:},{:}\!-\!1] - M[:,:,{:}\!-\!1,1{:}]$ 
\State $G[i,j]\leftarrow \mathsf{PDESolve}(M[i,j]),\quad \forall i\in\{1,\ldots,m\},~ j\in\{1,\ldots,n\}$ 
\State {\bfseries if} $\mathsf{full}$ {\bfseries then return} $G$ {\bfseries else}  {\bfseries return} $G[:,:,-1,-1]$
\end{algorithmic}
\end{algorithm}

\begin{algorithm}
\caption{$\mathsf{FirstOrderMMD}$ \hfill $\mathcal{O}(dm^2P^2)$}
\label{alg:firstordermmd}
\begin{algorithmic}[1]
\State {\bfseries Input:} sample paths $\{x^i\}_{i=1}^{m}\sim X$ and $\{y^j\}_{j=1}^{n}\sim Y$
\State {\bfseries Output:} an empirical estimator of the $1^{\text{st}}$ order MMD between $X$ and $Y$
\vspace{5pt}
\State $G^1_{XX}\leftarrow \mathsf{FirstOrderGram}(\{x^i\}_{i=1}^{m},\{x^i\}_{i=1}^{m},\mathsf{full}=\mathrm{False})$
\State $G^1_{XY}\leftarrow \mathsf{FirstOrderGram}(\{x^i\}_{i=1}^{m},\{y^j\}_{j=1}^{n},\mathsf{full}=\mathrm{False})$
\State $G^1_{YY}\leftarrow \mathsf{FirstOrderGram}(\{y^j\}_{j=1}^{n},\{y^j\}_{j=1}^{n},,\mathsf{full}=\mathrm{False})$
\vspace{5pt}
\State {\bfseries return} $\mathrm{avg}(G^1_{XX}) -2*\mathrm{avg}(G^1_{XY}) + \mathrm{avg}(G^1_{YY})$
\end{algorithmic}
\end{algorithm}

\subsection{Algorithm for the $2^{\text{nd}}$ order MMD}\label{ssec:second_order}
In the main paper, we derived the following estimator of the $2^{\text{nd}}$ order MMD,
\begin{equation*}
    \widehat{\mathcal{D}}^2_{\mathcal{S}}(X,Y) = \frac{1}{m(m-1)} \sum_{\substack{i,j=1 \\ i\neq j}}^m k_{\mathcal{S}}(\widetilde{x}^i,\widetilde{x}^j) - \frac{2}{mn} \sum_{i,j=1}^{m,n} k_{\mathcal{S}}(\widetilde{x}^i,\widetilde{y}^j) + \frac{1}{n(n-1)} \sum_{\substack{i,j=1 \\ i\neq j}}^n k_{\mathcal{S}}(\widetilde{y}^i,\widetilde{y}^j)
\end{equation*}
Compared to the $1^{\text{st}}$ order MMD, in order to compute this estimator, as outlined in \Cref{alg:secondordermmd}, we need to form the following three Gram matrices  $\mathbf{G}^{2}_{X,X}\in\mathbb{R}^{m\times m}$, $\mathbf{G}^{2}_{X,Y}\in\mathbb{R}^{m\times n}$ and $\mathbf{G}^{2}_{Y,Y}\in\mathbb{R}^{n\times n}$, 
\begin{equation*}
    [\mathbf{G}^{2}_{X,X}]_{i,j}=k_{\mathcal{S}}(\widetilde{x}^i,\widetilde{x}^j) \quad   [\mathbf{G}^{2}_{X,Y}]_{i,j}=k_{\mathcal{S}}(\widetilde{x}^i,\widetilde{y}^j) \quad   [\mathbf{G}^{2}_{Y,Y}]_{i,j}=k_{\mathcal{S}}(\widetilde{y}^i,\widetilde{y}^j)
\end{equation*}
As outlined in \Cref{alg:secondordergram}, this consists in two steps. Taking $\mathbf{G}^{2}_{X,Y}$ for example, first one forms $m\times n$ matrices of size $(P-1)\times (Q-1)$ each of the form, 
\begin{equation*}
    [M]_{p,q} = \langle \widetilde{x}_{s_{p+1}}-\widetilde{x}_{s_p},\widetilde{y}_{t_{q+1}}-\widetilde{y}_{t_q}\rangle_{\mathcal{H}_{\mathcal{S}}(V)}
\end{equation*}
(see \Cref{alg:innerprodpredcondKME}) and then one solves $m\times n$ PDEs with \Cref{alg:pdesolve}. This is summarized in \Cref{alg:secondordermmd}, which has time complexity $\mathcal{O}((d+m)m^2P^2)$ where $d$ is the number of coordinates of the paths $x$ and $y$. 
\begin{algorithm}
\caption{$\mathsf{SecondOrderGram}$ \hfill $\mathcal{O}(P^2m^3)$}
\label{alg:secondordergram}
\begin{algorithmic}[1]
\State {\bfseries Input:} $G_{XX}, G_{XY}, G_{YY}$ with 
$G_{XY}[i,j,p,q]=k_{\mathcal{S}}(x^i_{[0,s_p]},y^j_{[0,t_q]})$ and hyperparameter $\lambda$.
\State {\bfseries Output:} an empirical estimator of $G^2_{X,Y}\in\mathbb{R}^{m\times n}$, where $G^2_{X,Y}[i,j]=k_{\mathcal{S}}(\widetilde{x}^i,\widetilde{y}^j)$
\vspace{5pt}
\State $M \leftarrow\mathsf{InnerProdPredCondKME}(G_{XX},G_{XY},G_{YY},\lambda)$
\State $M\leftarrow M[:,:,1{:},1{:}] + M[:,:,{:}\!-\!1,{:}\!-\!1] - M[:,:,1{:},{:}\!-\!1] - M[:,:,{:}\!-\!1,1{:}]$
\State $G^2_{XY}[i,j]\leftarrow \mathsf{PDESolve}(M[i,j]),\quad \forall i\in\{1,\ldots,m\},~\forall j\in\{1,\ldots,n\}$ 
\State {\bfseries return} $G^2_{XY}$
\end{algorithmic}
\end{algorithm}

\begin{algorithm}
\caption{$\mathsf{SecondOrderMMD}$  \hfill $\mathcal{O}(dm^2P^2+P^2m^3)$}
\label{alg:secondordermmd}
\begin{algorithmic}[1]
\State {\bfseries Input:} sample paths $\{x^i\}_{i=1}^{m}\sim X$ and $\{y^j\}_{j=1}^{n}\sim Y$, hyperparameter $\lambda$.
\State {\bfseries Output:} an empirical estimator of the $2^{\text{nd}}$ order MMD between $X$ and $Y$
\vspace{5pt}
\State $G^1_{XX}\leftarrow \mathsf{FirstOrderGram}(\{x^i\}_{i=1}^{m},\{x^i\}_{i=1}^{m},\mathsf{full}=\mathrm{True})$
\State $G^1_{XY}\leftarrow \mathsf{FirstOrderGram}(\{x^i\}_{i=1}^{m},\{y^j\}_{j=1}^{n},\mathsf{full}=\mathrm{True})$
\State $G^1_{YY}\leftarrow \mathsf{FirstOrderGram}(\{y^j\}_{j=1}^{n},\{y^j\}_{j=1}^{n},\mathsf{full}=\mathrm{True})$
\vspace{5pt}
\State $G^2_{XX}\leftarrow \mathsf{SecondOrderGram}(G^1_{XX},G^1_{XX},G^1_{XX},\lambda)$
\State $G^2_{XY}\leftarrow \mathsf{SecondOrderGram}(G^1_{XX},G^1_{XY},G^1_{YY},\lambda)$
\State $G^2_{YY}\leftarrow \mathsf{SecondOrderGram}(G^1_{YY},G^1_{YY},G^1_{YY},\lambda)$
\vspace{5pt}
\State {\bfseries return} $\mathrm{avg}(G^2_{XX}) -2*\mathrm{avg}(G^2_{XY}) + \mathrm{avg}(G^2_{YY})$
\end{algorithmic}
\end{algorithm}
\begin{algorithm}[H]
\caption{$\mathsf{InnerProdPredCondKME}$ \hfill $\mathcal{O}(P^2m^3)$}
\label{alg:innerprodpredcondKME}
\begin{algorithmic}[1]
\State {\bfseries Input:} three Gram matrices $G_{XX}, G_{XY}, G_{YY}$ and hyperparameter $\lambda$
\State {\bfseries Output:} returns an empirical estimator of $M\in\mathbb{R}^{m\times n\times P\times Q}$ where $M[i,j,p,q]=\langle\widetilde{x}^i_{s_p},\widetilde{y}^j_{t_q}\rangle$
\State $W_X[:,:,p]\leftarrow (G_{XX}[:,:,p,p]+m\lambda I)^{-1}, \ \ \forall p\in\{1,\ldots,P\}$ 
\State $W_Y[:,:,q]\leftarrow (G_{YY}[:,:,q,q]+n\lambda I)^{-1}, \ \ \forall q\in\{1,\ldots,Q\}$ 
\For{$p$ from $1$ to $P$}
\For{$q$ from $1$ to $Q$}
\State $M[:,:,p,q]\leftarrow G_{XX}[:,:,p,p]^T W_X[:,:,p] G_{XY}[:,:,-1,-1] W_Y[:,:,q] G_{YY}[:,:,q,q]$
\EndFor
\EndFor
\end{algorithmic}
\end{algorithm}

\subsection{Algorithm for higher order MMDs}\label{ssec:higher_order_algorithm}
Now, we generalize the procedure in \Cref{ssec:second_order} for computing an estimator of $\mathcal{D}_{\mathcal{S}}^{k+1}$ when $k>1$,
\begin{equation*}
    \widehat{\mathcal{D}}^{k+1}_{\mathcal{S}}(X,Y) = \frac{1}{m(m-1)} \sum_{\substack{i,j=1 \\ i\neq j}}^m k_{\mathcal{S}}(\widetilde{x}^{k,i},\widetilde{x}^{k,j}) - \frac{2}{mn} \sum_{i,j=1}^{m,n} k_{\mathcal{S}}(\widetilde{x}^{k,i},\widetilde{y}^{k,j}) + \frac{1}{n(n-1)} \sum_{\substack{i,j=1 \\ i\neq j}}^n k_{\mathcal{S}}(\widetilde{y}^{k,i},\widetilde{y}^{k,j}),
\end{equation*}
where $\widetilde{x}^{k,i}$ and $\widetilde{y}^{k,j}$ denote sample paths from the processes $\mu^{k}_{X|\mathcal{F}_X}$ and $\mu^{k}_{Y|\mathcal{F}_Y}$ respectively. In order to compute this estimator, as outlined in \Cref{alg:secondordermmd}, we need to form the following three Gram matrices  $\mathbf{G}^{k+1}_{X,X}\in\mathbb{R}^{m\times m}$, $\mathbf{G}^{k+1}_{X,Y}\in\mathbb{R}^{m\times n}$ and $\mathbf{G}^{k+1}_{Y,Y}\in\mathbb{R}^{n\times n}$, 
\begin{equation*}
    [\mathbf{G}^{k+1}_{X,X}]_{i,j}=k_{\mathcal{S}}(\widetilde{x}^{k,i},\widetilde{x}^{k,j}) \quad   [\mathbf{G}^{k+1}_{X,Y}]_{i,j}=k_{\mathcal{S}}(\widetilde{x}^{k,i},\widetilde{y}^{k,j}) \quad   [\mathbf{G}^{k+1}_{Y,Y}]_{i,j}=k_{\mathcal{S}}(\widetilde{y}^{k,i},\widetilde{y}^{k,j})
\end{equation*}
As outlined in \Cref{alg:secondordergram}, this consists in two steps. Taking $\mathbf{G}^{k+1}_{X,Y}$ for example, first one forms $m\times n$ matrices of size $(P-1)\times (Q-1)$ each of the form, 
\begin{equation*}
    [M]_{p,q} = \langle \widetilde{x}^{k}_{s_{p+1}}-\widetilde{x}^{k}_{s_p},\widetilde{y}^{k}_{t_{q+1}}-\widetilde{y}^{k}_{t_q}\rangle_{\mathcal{H}^{k}_{\mathcal{S}}(V)}
\end{equation*}
(see \Cref{alg:innerprodpredcondKME}) and then one solves $m\times n$ PDEs with \Cref{alg:pdesolve}. This is summarized in \Cref{alg:secondordermmd}, which has time complexity $\mathcal{O}((d+km)m^2P^2)$ where $d$ is the number of coordinates of the paths $x$ and $y$.

\begin{algorithm}
\caption{$\mathsf{HigherOrderGram}$ \hfill $\mathcal{O}(P^2m^3)$}
\label{alg:secondordergram}
\begin{algorithmic}[1]
\State {\bfseries Input:} $G^{k}_{XX}, G^{k}_{XY}, G^{k}_{YY}$ with 
$G^{k}_{XY}[i,j,p,q]=k_{\mathcal{S}}(\widetilde{x}^{k-1,i}_{[0,s_p]},\widetilde{y}^{k-1,j}_{[0,t_q]})$ and hyperparameter $\lambda$.
\State {\bfseries Output:} an estimator of $G^{k+1}_{X,Y}\in\mathbb{R}^{m\times n\times P\times Q}$, where $G^{k+1}_{X,Y}[i,j,p,q]=k_{\mathcal{S}}(\widetilde{x}^{k,i}_{[0,s_p]},\widetilde{y}^{k,j}_{[0,t_q]})$
\vspace{5pt}
\State $M \leftarrow\mathsf{InnerProdPredCondKME}(G^{k}_{XX},G^{k}_{XY},G^{k}_{YY},\lambda)$
\State $M\leftarrow M[:,:,1{:},1{:}] + M[:,:,{:}\!-\!1,{:}\!-\!1] - M[:,:,1{:},{:}\!-\!1] - M[:,:,{:}\!-\!1,1{:}]$
\State $G^{k+1}_{XY}[i,j]\leftarrow \mathsf{PDESolve}(M[i,j],\mathsf{full}=\mathrm{True}),\quad \forall i\in\{1,\ldots,m\},~\forall j\in\{1,\ldots,n\}$ 
\State {\bfseries return} $G^{k+1}_{XY}$
\end{algorithmic}
\end{algorithm}

\begin{algorithm}
\caption{$\mathsf{HigherOrderMMD}$  \hfill $\mathcal{O}(dm^2P^2+(k-1)P^2m^3)$}
\label{alg:secondordermmd}
\begin{algorithmic}[1]
\State {\bfseries Input:} sample paths $\{x^i\}_{i=1}^{m}\sim X$ and $\{y^j\}_{j=1}^{n}\sim Y$, hyperparameter $\lambda$, order $k$.
\State {\bfseries Output:} an empirical estimator of the $k^{\text{th}}$ order MMD between $X$ and $Y$
\vspace{5pt}
\State $G_{XX}\leftarrow \mathsf{FirstOrderGram}(\{x^i\}_{i=1}^{m},\{x^i\}_{i=1}^{m},\mathsf{full}=\mathrm{True})$
\State $G_{XY}\leftarrow \mathsf{FirstOrderGram}(\{x^i\}_{i=1}^{m},\{y^j\}_{j=1}^{n},\mathsf{full}=\mathrm{True})$
\State $G_{YY}\leftarrow \mathsf{FirstOrderGram}(\{y^j\}_{j=1}^{n},\{y^j\}_{j=1}^{n},\mathsf{full}=\mathrm{True})$
\vspace{5pt}
\For{$\text{order}$ from $2$ to $k$}
\State $G^{\mathsf{new}}_{XX}\leftarrow \mathsf{HigherOrderGram}(G_{XX},G_{XX},G_{XX},\lambda)$
\State $G^{\mathsf{new}}_{XY}\leftarrow \mathsf{HigherOrderGram}(G_{XX},G_{XY},G_{YY},\lambda)$
\State $G^{\mathsf{new}}_{YY}\leftarrow \mathsf{HigherOrderGram}(G_{YY},G_{YY},G_{YY},\lambda)$
\State $G_{XX}, G_{XY}, G_{YY} \leftarrow G^{\mathsf{new}}_{XX}, G^{\mathsf{new}}_{XY}, G^{\mathsf{new}}_{YY} $
\EndFor
\vspace{5pt}
\State {\bfseries return} $\mathrm{avg}(G_{XX}[:,:,-1,-1]) -2*\mathrm{avg}(G_{XY}[:,:,-1,-1]) + \mathrm{avg}(G_{YY}[:,:,-1,-1])$
\end{algorithmic}
\end{algorithm}

\section{Experimental details}\label{sec:experimentaldetails}
We start with further experimental details for the applications of higher order distribution regression to quantitative finance (\Cref{sec:dr_finance}), where we consider the problem of optimally stopping fractional Brownian motions with different hurst exponents. 

\subsection{Rough volatility}

Rough volatility models constitute a class of models that are empirically well-tailored to fit observed implied market volatilities in the context of option pricing for short maturity assets. The basic model for option pricingis called the Black-Scholes model in which the volatility is assumed to be constant. Stochastic volatility models are extensions of the Black-Scholes model to the case where the volatility is itself stochastic. The main shortcoming of such stochastic volatility models is that they are able to capture the true steepness of the implied volatility smile close to maturity (see \cite{bayer2016pricing} for extra details). This is where rough volatility models become useful. Among them, the rough Bergomi model introduced by \cite{bayer2016pricing}, stood out for its ability to explain implied volatility and other phenomena related to European options.

\subsection{Higher order distribution regression}
\paragraph{Data} We use the data generator from \url{https://github.com/HeKrRuTe/OptStopRandNN} to simulate sample paths from $X$ a fractional Brownian motion (fBm) and obtain the solution of the optimal stopping time problem $\sup_{\tau}\mathbb{E}[g(X_\tau)|X_0]$. We note that although fBm is not typically used as a stock price model in quantitative finance, it is nevertheless considered a respected challenging example for optimal stopping algorithms \cite{herrera2021optimal,becker2019deep}. 

\paragraph{Models} We use a kernel Ridge regressor with different distribution regression kernels. Each is of the form $K(X,Y) = \exp(-\mathcal{D}(X,Y)^2/\sigma^2)$ where $\mathcal{D}(X,Y)$ is a maximum mean discrepancy. The models $K^1_\mathcal{S}$ and $K^2_\mathcal{S}$ correspond to the $1^{\text{st}}$ and $2^{\text{nd}}$ order maximum mean discrepancies $\mathcal{D}^1_{\mathcal{S}}$ and  $\mathcal{D}^2_{\mathcal{S}}$. We consider two other baselines (Mat\'{e}rn and RBF) for which the MMD is computed using the Matern 3/2 covariance function $k_{\text{mat32}}$, and the RBF covariance function $k_{\text{rbf}}$,
\begin{align*}
    k_{\text{mat32}}(x,y)=\left(1+\frac{\sqrt{3}}{\gamma^2}\norm{x-y}\right)\exp\left(-\frac{\sqrt{3}}{\gamma^2}\norm{x-y}\right), \quad k_{\text{rbf}}(x,y) = \exp\left(-\frac{\norm{x-y}^2}{\gamma^2}\right)
\end{align*}
All models are run $3$ times. 
The hyperparameters of all models are selected by cross-validation via a grid search on the training set ($70\%$ of the data selected at random) of each run. 

\subsubsection{Inferring causal graph for interacting bodies}
We provide further details for the last application (\Cref{sec:causality}) where the task is to infer whether any two bodies are connected by a spring from multiple observations of their 2D trajectories.  

\paragraph{Data}
We adapt the Pymunk simulator from \cite{li2020causal} publicly available at \url{https://github.com/pairlab/v-cdn}. For each pair of balls, there is a one-half probability that they are connected by nothing, or a spring. For each graph we run multiple episodes each of $20$ time steps. At the beginning of each episode,
we randomly assign the balls in different positions. The stiffness of the spring
relation is set to $20$, and we randomly sample the rest length between $[20, 120]$. 

\paragraph{Causal discovery algorithm} The PC algorithm \cite{spirtes2000causation} uses conditional independence tests to generate a causal graph from a dataset. The PC algorithm consists in two stages. The first stage, referred to as the \textit{skeleton phase}, consists in finding the structure of the causal graph. In the second stage, the edges are oriented by repetitively applying orientation rules. In the multi-body interaction example, we only need to perform the skeleton phase, which is sketched hereafter,
\begin{enumerate}
    \item Start with a complete graph 
    \item For each $X$ and $Y$ which are still connected. If there is a third variable $Z_1$ connected to $X$ or $Y$, such that
$X \indep Y\mid Z_1$, remove the edge between $X$ and $Y$.  
\item For each $X$ and $Y$ which are still connected, if there is a third and a fourth variable 
$Z_1$ and $Z_2$ connected to $X$ or $Y$ such that  $X \indep Y\mid Z_1,Z_2$, remove the edge between $X$ and $Y$.  
\item  Iteratively increase the cardinality of the set of variables on which to condition.
\end{enumerate}

To test for conditional independence we use the Hilbert-Schmidt conditional independence criterion $H_{XY|Z}$ for stochastic processes (\Cref{ssec:HSICstopro}). The combination of the PC algorithm with a kernel-based dependence measure has been used in \cite{sun2007kernel} and \cite{tillman2009nonlinear} where it is termed kPC. However, to our knowledge it has never been used in conjunction with a kernel-based measure of dependence for multidimensional stochastic processes.

Since the null distribution of the test statistics $H_{YX|Z}$ is not known, one possibility would be to use a permutation approach as in \cite[Sec 2]{tillman2009nonlinear}. However the latter is not computationally efficient. We leave the development of a faster approach for future work, and adopt the approach \cite{sun2007kernel} for this experiment. That is we use a threshold $\alpha$ and remove an edge if there is a $Z$ such that $H_{XY|Z}<\alpha$. We repeat $15$ times the causal discovery procedure and use $30\%$ of the runs to fix $\alpha$.

All experiments in \Cref{sec:applications} have been run on a P100 GPU to leverage an efficient dedicated CUDA implementation of the signature kernel. 

\section{Proofs}\label{sec:proofs}

\begin{theorem}\label{theorem: rank 2 MMD}
Given two stochastic processes $X,Y$
\begin{align*}
    \mathcal{D}^2_{\mathcal{S}}(X,Y)  = 0 
    \iff \mathbb{P}_{X\mid \mathcal{F}_{X}}  = \mathbb{P}_{Y \mid \mathcal{F}_Y} \quad
\end{align*}
Furthermore
\begin{equation*}
    \mathcal{D}^2_{\mathcal{S}}(X,Y) = 0 \implies \mathcal{D}^1_{\mathcal{S}}(X,Y) = 0
\end{equation*}
but the converse is not generally true.
\end{theorem}

\begin{proof}
First we note that by a standard result in signature kernel learning theory. e.g., \cite{chevyrev2018signature}, for $X \in \mathcal{X}(V)$ and every $t$, the mapping
$$
\mathbb{P}_{X|\mathcal{F}_{X_t}} \mapsto \mu^1_{X|\mathcal{F}_{X_t}} = \int k_{\mathcal{S}}(\cdot,x)\mathbb{P}_{X|\mathcal{F}_{X_t}}(dx)
$$
is a homeomorphism (with respect to weak topology and Hilbert space topology); in particular, we have
$$
\mathbb{P}_{X|\mathcal{F}_X} = \mathbb{P}_{Y|\mathcal{F}_Y} \iff \mathbb{P}_{\mu^1_{X|\mathcal{F}_{X}}} = \mathbb{P}_{\mu^1_{Y|\mathcal{F}_{Y}}}.
$$
Then, using the same argument for $\mathbb{P}_{\mu^1_{X|\mathcal{F}_{X}}}$ and $\mathbb{P}_{\mu^1_{Y|\mathcal{F}_{Y}}}$, we can  further deduce that
$$
\mathbb{P}_{\mu^1_{X|\mathcal{F}_{X}}} = \mathbb{P}_{\mu^1_{Y|\mathcal{F}_{Y}}} \iff \int k_{\mathcal{S}}(\cdot,x)\mathbb{P}_{\mu^1_{X|\mathcal{F}_X}}(dx) (= \mu^2_X) =  \int k_{\mathcal{S}}(\cdot,y)\mathbb{P}_{\mu^1_{Y|\mathcal{F}_Y}}(dy)(= \mu^2_Y).
$$
Since by definition it holds that $\mathcal{D}^2_{\mathcal{S}}(X,Y) = \|\mu^2_X - \mu^2_Y\|_{\mathcal{H}^2(V)}$, we complete the proof of the first claim in this theorem.

For the second claim, it is easy to see that by definition $\mathbb{P}_{X|\mathcal{F}_X} = \mathbb{P}_{Y|\mathcal{F}_Y}$ ensures that $\mathbb{P}_X = \mathbb{P}_Y$, and therefore the implication that $\mathcal{D}^2_{\mathcal{S}}(X,Y) = 0 \implies \mathcal{D}^1_{\mathcal{S}}(X,Y) = 0$ follows immediately from the fact that $\mathbb{P}_X = \mathbb{P}_Y \iff \mathcal{D}^1_{\mathcal{S}}(X,Y) = 0$. Moreover, we refer readers to \cite[Example 3.1]{Hoover84adapt} for a simple example which shows that there exist processes $X$ and $Y$ with $\mathcal{D}^1_{\mathcal{S}}(X,Y) = 0$ but $\mathcal{D}^2_{\mathcal{S}}(X,Y) > 0$.
\end{proof}

\begin{theorem}
$\widehat{\mathcal{D}}^2_{\mathcal{S}}(X,Y)$ is a consistent estimator for the $2^{\text{nd}}$ order MMD, i.e.
\begin{equation}
    |\widehat{\mathcal{D}}^2_{\mathcal{S}}(X,Y) - \mathcal{D}^2_{\mathcal{S}}(X,Y)| \overset{p}{\to} 0 \quad \text{ as } \ m,n \to\infty
\end{equation}
with $\{x^i\}_{i=1}^m \sim X$, $\{y^i\}_{i=1}^n \sim Y$ and where convergence is in probability.
\end{theorem}

\begin{proof}
Recall that given $m$ independent sample paths $\{x^i\}_{i=1}^m \sim X$, we can use the estimator in \cref{eqn:estimator_x} to approximate sample paths $\{\widetilde{x}^i\}_{i=1}^m$ from the $1^{\text{st}}$ order predictive KME $\mu^1_{X \mid \mathcal{F}_X}$. Hence, it suffices to prove the following claim.

\paragraph{Claim:} consider $m$ independent sample paths $\{\widetilde{x}^i\}_{i=1}^{m} \sim \mu^1_{X \mid \mathcal{F}_X}$. Then, the estimator given by $\widehat{\mu}^2_X = \frac{1}{m}\sum_{i=1}^mk_\mathcal{S}(\widetilde{x}^i, \cdot)$ is consistent for the $2^{\text{nd}}$ order predictive KME $\mu^2_X$, i.e. 
\begin{equation}\label{eq: convergence of empirical estimates}
    \norm{\widehat{\mu}^2_X - \mu^2_X}_{\mathcal{H}^{ 2}_{\mathcal{S}}(V)}^2 \overset{p}{\to} 0, \text{ as } m \to \infty
\end{equation}

By the triangular inequality 
\begin{align}
\norm{\widehat{\mu}^2_X - \mu^2_X}_{\mathcal{H}^{ 2}_{\mathcal{S}}(V)}^2 &=  \norm{ \frac{1}{m}\sum_{i=1}^m k_{\mathcal{S}}(\widetilde{x}^i,\cdot) - \mathbb{E}[k_{\mathcal{S}}(\mu_{X \mid \mathcal{F}_X}^1,\cdot)]}_{\mathcal{H}^{ 2}_{\mathcal{S}}(V)}^2\\
& \leq \norm{ \frac{1}{m}\sum_{i=1}^{m} k_{\mathcal{S}}(\widetilde{x}^i,\cdot) - \mathbb{E}[k_{\mathcal{S}}(\widehat{\mu}_{X \mid \mathcal{F}_X}^1,\cdot)]}_{\mathcal{H}^{ 2}_{\mathcal{S}}(V)}^2 \label{eqn:28}\\
& + \norm{ \mathbb{E}[k_{\mathcal{S}}(\widehat{\mu}_{X \mid \mathcal{F}_X}^1,\cdot)] - \mathbb{E}[k_{\mathcal{S}}(\mu_{X \mid \mathcal{F}_X}^1,\cdot)]}_{\mathcal{H}^{ 2}_{\mathcal{S}}(V)}^2
\end{align}
The term in (\ref{eqn:28}) converges to $0$ as $m\to\infty$ by the weak law of large numbers. Therefore, it remains to show that 
\begin{equation}
    \norm{ \mathbb{E}_{X}[k_{\mathcal{S}}(\widehat{\mu}_{X \mid \mathcal{F}_X}^1,\cdot)] - \mathbb{E}_{X}[k_{\mathcal{S}}(\mu_{X \mid \mathcal{F}_X}^1,\cdot)]}_{\mathcal{H}^2_{\mathcal{S}}(V)} \overset{p}{\to} 0, \text{ as } m \to \infty
\end{equation}
First note the following upper bound
\begin{equation*}
    \norm{ \mathbb{E}_{X}[k_{\mathcal{S}}(\widehat{\mu}_{X \mid \mathcal{F}_X}^1,\cdot)] - \mathbb{E}_{X}[k_{\mathcal{S}}(\mu_{X \mid \mathcal{F}_X}^1,\cdot)]}_{\mathcal{H}^{ 2}_{\mathcal{S}}(V)} \leq \mathbb{E}_{X}\norm{k_{\mathcal{S}}(\widehat{\mu}_{X \mid \mathcal{F}_X}^1,\cdot) - k_{\mathcal{S}}(\mu_{X \mid \mathcal{F}_X}^1,\cdot)}_{\mathcal{H}^{ 2}_{\mathcal{S}}(V)}
\end{equation*}
We will show convergence of the right-hand-side. By \cite[Theorem 3.4]{park2021LS}, for every $t=1,\ldots,T$
\begin{equation}\label{eq: convergence in L2 norm}
    \mathbb{E}_{X|\mathcal{F}_{X_t}}\norm{\widehat{\mu}_{X \mid \mathcal{F}_{X_t}}^1 - \mu^1_{X\mid\mathcal{F}_{X_t}} }_{\mathcal{H}_{\mathcal{S}}(V)}^2 \overset{p}{\to} 0 \quad \text{ as }m\to\infty
\end{equation}
Now let us assume that the above convergences also hold almost surely for every $t =1,\ldots,T$. Then by the Egorov's theorem, for any $\delta>0$, there is a subset $\Omega_\delta$ with $\mathbb{P}(\Omega_\delta) > 1 - \delta$ and the above convergence \eqref{eq: convergence in L2 norm} holds uniformly on $\Omega_\delta$.
This implies that on $\Omega_\delta$ for every $\varepsilon > 0$ there is an $N(\varepsilon)$ such that for all $m \ge N(\varepsilon)$ and all $t = 1,\ldots,T$, it holds that
\begin{equation}
    \mathbb{E}_{X|\mathcal{F}_{X_t}}\norm{\widehat{\mu}_{X \mid \mathcal{F}_{X_t}}^1 - \mu^1_{X\mid\mathcal{F}_{X_t}} }_{\mathcal{H}_{\mathcal{S}}(V)}^2 \le \varepsilon^2
\end{equation}
From this estimate we immediately obtain by the triangle inequality on $\Omega_\delta$
\begin{equation}\label{eq: upper bound on norm}
\mathbb{E}_{X|\mathcal{F}_{X_t}}\norm{\widehat{\mu}_{X \mid \mathcal{F}_{X_t}}^1 }_{\mathcal{H}_{\mathcal{S}}(V)}^2 \le 2 \mathbb{E}_{X|\mathcal{F}_{X_t}}\norm{\mu^1_{X\mid\mathcal{F}_{X_t}} }_{\mathcal{H}_{\mathcal{S}}(V)}^2+ 2\varepsilon^2,
\end{equation}
and, by the Chebyshev's inequality, on $\Omega_\delta$, $\forall t = 1,\ldots,T$, $\forall m \ge N(\varepsilon)$
\begin{align*}
    \mathbb{P}_{X|\mathcal{F}_{X_t}}\left[\norm{\widehat{\mu}_{X \mid \mathcal{F}_{X_t}}^1 - \mu^1_{X\mid\mathcal{F}_{X_t}} }_{\mathcal{H}_{\mathcal{S}}(V)} > \sqrt\varepsilon \right] \leq \frac{1}{\varepsilon}\mathbb{E}_{X|\mathcal{F}_{X_t}}\norm{\widehat{\mu}_{X \mid \mathcal{F}_{X_t}}^1 - \mu^1_{X\mid\mathcal{F}_{X_t}} }_{\mathcal{H}_{\mathcal{S}}(V)}^2 \leq \frac{1}{\varepsilon}\varepsilon^2 = \varepsilon,
\end{align*}
which implies that on $\Omega_\delta$, $\forall t = 1,\ldots,T$, the sequence $\widehat{\mu}_{X \mid \mathcal{F}_{X_t}}^1$ converges to $\mu^1_{X\mid\mathcal{F}_{X_t}}$ in probability with respect to $\mathbb{P}_X$. By a standard result in rough path theory \cite{lyons1998differential} there exists a universal constant $\beta \in \mathbb{R}$ such that
\begin{equation}
    \norm{k_{\mathcal{S}}(\widehat{\mu}^1_{X \mid \mathcal{F}_X},\cdot)}_{\mathcal{H}^{ 2}_{\mathcal{S}}(V)} \le \beta \norm{\widehat{\mu}_{X \mid \mathcal{F}_X}^1}_{\mathcal{H}_{\mathcal{S}}(V)}^{1\text{--var}}
\end{equation}
where $\norm{\cdot}_{\mathcal{H}_{\mathcal{S}}(V)}^{1\text{--var}}$ denotes the total variation norm of paths taking values in $\mathcal{H}_{\mathcal{S}}(V)$. Since we are in a finite discrete time setup, it is easy to see that 
\begin{equation}
    \norm{\widehat{\mu}_{X \mid \mathcal{F}_X}^1}_{\mathcal{H}_{\mathcal{S}}(V)}^{1\text{--var}} \le C(T) \sum_{t=1}^T \norm{\widehat{\mu}_{X \mid \mathcal{F}_{X_t}}^1}_{\mathcal{H}_{\mathcal{S}}(V)}
\end{equation}
Hence, combining all above arguments, we can conclude that on $\Omega_\delta$ and for all $m \ge N(\varepsilon)$,
\begin{align}
    \mathbb{E}_{X}\norm{k_{\mathcal{S}}(\widehat{\mu}_{X \mid \mathcal{F}_X}^1,\cdot)}^2_{\mathcal{H}^{ 2}_{\mathcal{S}}(V)} &\le \beta^2 \mathbb{E}_X \norm{\widehat{\mu}^1_{X \mid \mathcal{F}_X}}^2_{1\text{--var};\mathcal{H}_{\mathcal{S}}(V)}\\
    &\le \beta^2 C(T) \sum_{t=1}^T \mathbb{E}_X\norm{\widehat{\mu}_{X \mid \mathcal{F}_{X_t}}^1 }_{\mathcal{H}_{\mathcal{S}}(V)}^2\\
    &\le \beta^2 C(T)  \Big(2\sum_{t=1}^T \mathbb{E}_{X}\norm{\mu^1_{X\mid\mathcal{F}_{X_t}} }_{\mathcal{H}_{\mathcal{S}}(V)}^2+ 2T\varepsilon^2\Big) \le C < \infty
\end{align}
where in the last line we used \eqref{eq: upper bound on norm}. As a result, we obtain that on $\Omega_\delta$, 
\begin{equation}
    \sup_{m \ge N(\varepsilon)} \mathbb{E}_{X}\norm{k_{\mathcal{S}}(\widehat{\mu}_{X \mid \mathcal{F}_X}^1,\cdot) - k_{\mathcal{S}}(\mu_{X \mid \mathcal{F}_X}^1,\cdot)}^2_{\mathcal{H}^{ 2}_{\mathcal{S}}(V)} < \infty
\end{equation}
which in turn implies, by the de la Vall\'ee--Poussin theorem, that on $\Omega_\delta$, the sequence $\norm{k_{\mathcal{S}}(\widehat{\mu}_{X \mid \mathcal{F}_X}^1,\cdot) - k_{\mathcal{S}}(\mu_{X \mid \mathcal{F}_X}^1,\cdot)}^2_{\mathcal{H}^{ 2}_{\mathcal{S}}(V)}$, $m \ge N(\varepsilon)$ is uniformly integrable for $\mathbb{P}_X$. Then recalling that we have shown that $\widehat{\mu}_{X \mid \mathcal{F}_X}^1$ converges to $\mu_{X \mid \mathcal{F}_X}^1$ in probability with respect to $\mathbb{P}_X$ as $m \to \infty$, a standard result in probability theory ensures that (thanks to the uniform integrability of the sequence and the continuity of the kernel $k_{\mathcal{S}}$ which ensures that $\norm{k_{\mathcal{S}}(\widehat{\mu}_{X \mid \mathcal{F}_X}^1,\cdot) - k_{\mathcal{S}}(\mu_{X \mid \mathcal{F}_X}^1,\cdot)}_{\mathcal{H}^{ 2}_{\mathcal{S}}(V)} \to 0$ in probability for $\mathbb{P}_X$)
on $\Omega_\delta$, $\mathbb{E}_{X}\norm{k_{\mathcal{S}}(\widehat{\mu}_{X \mid \mathcal{F}_X}^1,\cdot) - k_{\mathcal{S}}(\mu_{X \mid \mathcal{F}_X}^1,\cdot)}_{\mathcal{H}^{ 2}_{\mathcal{S}}(V)} \to 0, \text{ as } m \to \infty,
$ which implies that $\norm{ \mathbb{E}_{X}[k_{\mathcal{S}}(\widehat{\mu}_{X \mid \mathcal{F}_X}^1,\cdot)] - \mathbb{E}_{X}[k_{\mathcal{S}}(\mu_{X \mid \mathcal{F}_X}^1,\cdot)]}_{\mathcal{H}^{ 2}_{\mathcal{S}}(V)} \to 0, \text{ as } m \to \infty$, on $\Omega_\delta$. Clearly, as $\delta$ was arbitrary, we have $ \mathbb{E}_{X|\mathcal{F}_{X_t}}\norm{\widehat{\mu}_{X \mid \mathcal{F}_{X_t}}^1 - \mu^1_{X\mid\mathcal{F}_{X_t}} }_{\mathcal{H}_{\mathcal{S}}(V)}^2 \to 0 \quad \text{ as }m\to\infty$ a.s. Finally, note that the above result holds true for any subsequence of $(\widehat{\mu}_{X \mid \mathcal{F}_{X_t}}^1)_{m \ge 1}$ (because every sequence converging in probability has a subsequence converging almost surely), which proves the desired result \eqref{eq: convergence of empirical estimates}.
\end{proof}

\begin{theorem}\label{theorem: MMD n}
Given two stochastic processes $X,Y$
\begin{equation}
    \mathcal{D}^n_{\mathcal{S}}(X,Y) = 0 \implies \mathcal{D}^k_{\mathcal{S}}(X,Y) = 0 \quad \text{for any } \ 1<k<n
\end{equation}
but the converse is not generally true.
\end{theorem}

\begin{proof}
 Let $X \in \mathcal{X}(V)$, then we denote $\mathbb{P}_{X|\mathcal{F}_{X_t}} =: X^{(1)}_t$. Then we continue this procedure and define $X^{(n)}_t := \mathbb{P}_{X^{n-1}|\mathcal{F}_{X_t}}$ (it is called the rank $n$ prediction process in \cite{bonnier2020adapted}). Now we can apply the same argument as in the proof of Theorem \ref{theorem: rank 2 MMD} together with  an induction procedure easily prove that 
\begin{equation}
    \mathcal{D}^n_{\mathcal{S}}(X,Y) = 0  \iff \mathbb{P}_{X^{(n-1)}} = \mathbb{P}_{Y^{(n-1)}}
\end{equation}
for all $n > 1$. From the definition of these processes $X^{(n)}$ and $Y^{(n)}$ we can immediately see that $\mathbb{P}_{X^{(n)}} = \mathbb{P}_{Y^{(n)}}$ ensures that $\mathbb{P}_{X^{(k)}} = \mathbb{P}_{Y^{(k)}}$ for all $k < n$, which yields the desired result. We refer readers to \cite[Example 3.2]{Hoover84adapt} for examples which illustrate that for each $n$ there exist processes $X$ and $Y$ with $\mathcal{D}^{n}_{\mathcal{S}}(X,Y) = 0$ (equivalently, $\mathbb{P}_{X^{(n-1)}} = \mathbb{P}_{Y^{(n-1)}}$) but $\mathcal{D}^{n+1}_{\mathcal{S}}(X,Y) > 0$ (equivalently, $\mathbb{P}_{X^{(n)}} \neq \mathbb{P}_{Y^{(n)}}$).
\end{proof}

\begin{remark}
\normalfont
Using terminologies from \cite{Hoover84adapt} and \cite{bonnier2020adapted}, $\mathbb{P}_{X^{(n)}} = \mathbb{P}_{Y^{(n)}}$ means that processes $X$ and $Y$ have the same adapted distribution up to rank $n$. Therefore \Cref{theorem: rank 2 MMD} and \ref{theorem: MMD n} tell us that  $\mathcal{D}^n_{\mathcal{S}}(X,Y) = 0$ if and only if they ave the same adapted distribution up to rank $n$. Moreover, using the partial isometry between RKHS generated by $k_{\mathcal{S}}$ and tensor algebra space, see e.g. \cite[Theorem E.2]{chevyrev2018signature}, one can use an induction argument to verify that $\mathcal{D}^n_{\mathcal{S}}$ coincides with the metric $d_{n-1}$ defined in \cite[Definition 14]{bonnier2020adapted}, and therefore by \cite[Theorem 4]{bonnier2020adapted} we can obtain a stronger result that $\mathcal{D}^n_{\mathcal{S}}$ actually metrizes the so--called rank $n-1$ adapted topology (see \cite[Definition 5]{bonnier2020adapted}, \cite[Definition 2.25]{Hoover84adapt}). For more details regarding adapted topologies we refer to \cite{bonnier2020adapted}.
\end{remark}
\begin{theorem}
Let $f: \mathbb{R} \to \mathbb{R}$ be a globally analytic function with non-negative coefficients. Define the family of kernels $K^n_{\mathcal{S}}:\mathcal{P}(\mathcal{X}(V))\times\mathcal{P}(\mathcal{X}(V))\to\mathbb{R}$ as follows
\begin{equation}
K^n_{\mathcal{S}}(X,Y) = f(\mathcal{D}_\mathcal{S}^n(X,Y)), \quad n \in \mathbb{N}_{\geq 1}
\end{equation}
Then the RKHS associated to $K^n_{\mathcal{S}}$ is dense in the space of functions from $\mathcal{P}(\mathcal{X}(V))$ to $\mathbb{R}$ which are continuous with respect to the $k^{\text{th}}$ order MMD for any $1\leq k\leq n$.
\end{theorem}

\begin{proof}
By \cite[Thm. 2.2]{christmann2010universal} if $K$ is a compact metric space and $H$ is a separable Hilbert space such that there exists a continuous (w.r.t. a topology $\tau$ on $K$) and injective map $\rho : K \to H$, then for any globally analytic function with non-negative coefficients $f:\mathbb{R} \to \mathbb{R}$ the kernel $k: K \times K \to \mathbb{R}$ given by 
\begin{equation}
k(z,z') = f\left(\norm{\rho(z)-\rho(z')}_{H}\right)
\end{equation}
is universal in the sense that its RKHS is $\tau$-dense in the space of $\tau$-continuous functions from $K$ to $\mathbb{R}$. By assumption, $\mathcal{X}(V)$ is a $\mathcal{D}_\mathcal{S}^1$-compact metric space, therefore by \Cref{theorem: MMD n} it is also $\mathcal{D}_\mathcal{S}^n$-compact for every $n \geq 1$. Hence, by \cite[Thm. 10.2]{walkden2014ergodic} the set of stochastic processes $\mathcal{P}(\mathcal{X}(V))$ is also $\mathcal{D}_\mathcal{S}^n$-compact. 
For showing that $\rho : X \mapsto \mu^n_X$ is injective and continous with respect to $\mathcal{D}_\mathcal{S}^n$ we refer readers to \cite[Proposition 4]{bonnier2020adapted}, one only needs to verify that $\mu^n_X$ corresponds to the mapping $\bar S_n$ used in the proof of \cite[Proposition 4]{bonnier2020adapted} by the definition of $\mu^n_X$ and the fact that $\mathcal{D}_\mathcal{S}^n$ metrizes the rank $n-1$ adapted topology (cf. the above remark). 
Furthermore $\mathcal{H}_\mathcal{S}^n(V)$ can be shown by induction to be a Hilbert space with a countable basis, hence it is separable. Setting $K=\mathcal{P}(\mathcal{X}(V)), H = \mathcal{H}_\mathcal{S}^n(V)$ and $\rho:X \mapsto \mu^n_X$ concludes the proof. 
\end{proof}

\end{document}